\documentclass[lettersize,journal]{IEEEtran}

\usepackage{amsmath,amssymb,amsthm,mathrsfs,amsfonts,stmaryrd, wasysym,marvosym}
\usepackage{mathtools,multirow}
\usepackage{color}
\usepackage{graphicx,subcaption,cite,epsfig,dblfloatfix,url,xcolor} 
\usepackage{etoolbox}
\allowdisplaybreaks[4] 
\usepackage{caption}
\usepackage{makecell}
\usepackage{pxfonts}
\usepackage{dsfont}

\newtheorem{theorem}{\textbf{Theorem}}

\newtheorem{proposition}{\textbf{Proposition}}
\newtheorem{example}{\textbf{Example}}

\newtheorem{lemma}{\textbf{Lemma}}

\newtheorem{remark}{\textbf{Remark}}
\newcommand{\dv}{\mathbf} 
\newcommand{\bm}{\mathbf} 
\newcommand{\mc}{\mathcal} 
\newcommand{\mkv}{-\!\!\!\!\minuso\!\!\!\!-}

\graphicspath{{./}}

\newcommand*\xbar[1]{%
	\hbox{%
		\vbox{%
			\hrule height 0.5pt 
			\kern0.5ex
			\hbox{%
				\kern-0.1em
				\ensuremath{#1}%
				\kern-0.1em
			}%
		}%
	}%
} 

\usepackage{algorithm} 
\usepackage{algpseudocode}

\algnewcommand{\Inputs}[1]{%
	\State \textbf{input:}
	\parbox[t]{.8\linewidth}{\raggedright #1}
}
\algnewcommand{\Initialize}[1]{%
	\State \textbf{initialization}
	\parbox[t]{.95\linewidth}{\raggedright #1}
}
\algnewcommand{\Outputs}[1]{%
	\State \textbf{output:}
	\parbox[t]{.8\linewidth}{\raggedright #1}
}

\begin{document}
	\fontencoding{OT1}\fontsize{9.6}{11.25pt}\selectfont
	\title{In-Network Learning: Distributed Training and Inference in Networks}
	\author{Matei~Moldoveanu$^{\nmid}$ \qquad Abdellatif Zaidi$\:^{\dagger}$$\:^{\nmid}$\vspace{0.3cm}\\
		$^{\dagger}$ Universit\'e Paris-Est, Champs-sur-Marne 77454, France\\
		$^{\nmid}$ Mathematical and Algorithmic Sciences Lab., Paris Research Center, Huawei France \\
		\{matei.moldoveanu@huawei.com, abdellatif.zaidi@univ-eiffel.fr\}
	}

	\maketitle
	\pagenumbering{gobble}
	
	\begin{abstract}
	In this paper, we study inference and learning over networks that can be modeled by a directed graph. Specifically, nodes are equipped each with a (possibly different) neural network and some of them possess data that is relevant to some inference task which needs to be performed at the end (fusion) node, with the risk measured under logarithmic loss. The graph defining the network topology is fixed and known. We develop a learning algorithm and an architecture that make use of the multiple data streams and processing units available distributively, not only during the training phase but also during the inference phase. In particular, the analysis reveals how inference propagates and fuses across a network. We study the design criterion of our proposed method and its bandwidth requirements. Also, we discuss implementation aspects using neural networks in typical wireless radio access; and provide experiments that illustrate benefits over state-of-the-art techniques.
	\end{abstract}

	
	\section{Introduction}
	
	The unprecedented success of modern machine learning (ML) techniques in areas such as computer vision~\cite{obj_det_10_years_survey}, neuroscience~\cite{GLASER2019126}, image processing~\cite{medical_images_survey}, robotics~\cite{robotics_reinforced_learning} and natural language processing~\cite{VinyalsL15} has lead to an increasing interest for their application to wireless communication systems over the recent years. Early efforts along this line of work fall in what is sometimes referred to as the "learning to communicate" paradigm in which the goal is to automate one or more communication modules such as the modulator-demodulator, the channel coder-decoder, or others, by replacing them with suitable ML algorithms. Although important progress has been made for some particular communication systems, such as the molecular one~\cite{molecular_com}, it is still not clear yet whether ML techniques can offer a reliable alternate solution to model-based approaches, especially as typical wireless environments suffer from time-varying noise and interference. 
	
	Wireless networks have other important intrinsic features which may pave the way for more cross-fertilization between ML and communication, as opposed to applying ML algorithms as black boxes in replacement of one or more communication modules. For example, while in areas such as computer vision, neuroscience, and others, relevant data is generally available at one point, it is typically highly distributed across several nodes in wireless networks. Examples include amplitude or phase information or the so-called radio-signal strength indicator (RSSI) of a user's signal, which can be used for localization purposes in fingerprinting-based approaches~\cite{CSI_fingerprinting}, and are typically available at several base stations. A prevalent approach for the implementation of ML solutions in such cases would consist in collecting all relevant data at one point (a cloud server) and then train a suitable ML model using all available data and processing power. Because the volumes of data needed for training are generally large, and with the scarcity of network resources (e.g., power and bandwidth), that approach might not be appropriate in many cases, however. In addition, some applications might have stringent latency requirements which are incompatible with sharing the data, such as in automatic vehicle driving. In other cases, it might be desired not to share the raw data for the sake of enhancing the privacy of the solution, in the sense that infringing the user's privacy is generally more easily accomplished from the raw data itself than from the output of a neural network (NN) that takes that data as input.
	
	The above has called for a new paradigm in which intelligence moves from the heart of the network to its edge, which is sometimes referred to as "Edge Learning". In this new paradigm, communication plays a central role in the design of efficient ML algorithms and architectures because both data and computational resources, which are the main ingredients of an efficient ML solution, are highly distributed. A key aspect towards building suitable ML-based solutions is whether the setting assumes only the training phase involves distributed data (sometimes referred to as distributed \textit{learning} such as the Federated Learning (FL) of~\cite{mcmahan2017,fedlearn2017}) or if the inference (or test) phase too involves distributed data. 
	
	There is a vast body of literature on problems related to distributed estimation and detection (see, e.g., ~\cite{xiao_distributed_2006, kreidl_decentralized_2010,chamberland_wireless_2007, Tsitsiklis93decentralizeddetection} and references therein). In particular, most related to this paper, a growing line of works focuses on developing distributed learning algorithms and architectures. Examples include \cite{simic_learning-theory_2003} and \cite{predd_distributed_2006} which use kernel methods and \cite{nguyen_nonparametric_2005} and \cite{jagyasi_data_2015} which use marginalized kernels and NNs, respectively. Perhaps most popular and related to our work, however, is the FL of~\cite{mcmahan2017,fedlearn2017} which, as we already mentioned, is most suitable for scenarios in which the training phase has to be performed distributively while the inference phase has to be performed centrally at one node. To this end, during the training phase nodes (e.g., base stations) that possess data are all equipped with copies of a single NN model which they simultaneously train on their locally available data-sets. The learned weight parameters are then sent to a cloud- or parameter server (PS) which aggregates them, e.g. by simply computing their average. The process is repeated, every time re-initializing using the obtained aggregated model, until convergence. The rationale is that, this way, the model is progressively adjusted to account for all variations in the data, not only those of the local data-set. For recent advances on FL and applications in wireless settings, the reader may refer to~\cite{TBZNHC019,M.-AG20,YLQP20} and references therein. Another relevant work is the Split Learning (SL) of~\cite{gupta2018distributed} in which, for a multiaccess type network topology, a two-part NN model that is split into an encoder part and a decoder part is learned sequentially. The decoder does not have its own data; and, in every round, the NN encoder part is fed with a distinct data-set and its parameters are initialized using those learned from the previous round. The learned two-part model is then used as follows during the inference: one part of this model is used by an encoder and the other one by a decoder. Another variation of SL, sometimes called "vertical SL", was proposed recently in \cite{ceballos_splitnn-driven_2020}. The approach uses vertical partitioning of the data; and, in the special case of a multi-access topology, it is similar to the in-network learning solution that we propose in this paper. 
	
	Compared to both SL and FL, which consider only the training phase to be distributed, in this paper we focus on the problem in which the inference phase also takes place distributively. More specifically, in this paper, we study a network inference problem in which some of the nodes possess each, or can acquire, part of the data that is relevant for inference on a random variable $Y$. The node at which the inference needs to be performed is connected to the nodes that possess the relevant data through a number of intermediate other nodes. We assume that the network topology is fixed and known. This may model, e.g., a setting in which a macro BS needs to make inference on the position of a user on the basis of summary information obtained from correlated CSI measurements $X_1,\hdots,X_J$ that are acquired at some proximity edge BSs. Each of the edge nodes is connected with the central node either directly, via an error free link of given finite capacity, or via intermediary nodes. While in some cases it might be enough to process only a subset of the $J$ nodes, we assume that processing only a (any) strict subset of the measurements cannot yield the desired inference accuracy; and, as such, the $J$ measurements $X_1,\hdots,X_J$ need to be processed during the inference or test phase.

	\begin{example}{(Autonomous Driving)}
		One basic requirement of the problem of autonomous driving is the ability to cope with problematic roadway situations, such as those involving construction, road hazards, hand signals and reckless drivers. Current approaches mostly rely on equipping the vehicle with more on-board sensors. Clearly, while this can only allow a better coverage of the navigation environment, it seems unlikely to successfully cope with the problem of blind spots due, e.g., to obstruction or hidden obstacles. In such contexts, external sensors such as other vehicles' sensors, cameras installed on the roofs of proximity buildings or wireless towers may help perform a more precise inference, by offering a complementary, possibly better, view of the navigation scene. An example scenario is shown in Figure~\ref{fig-autonomous-driving}. The application requires real-time inference which might be incompatible with current cellular radio standards, thus precluding the option of sharing the sensors' raw data and processing it locally, e.g., at some on-board server. When equipped with suitable intelligence capabilities each sensor can successfully identify and extract those features of its measurement data that are not captured by other sensors' data. Then, it only needs to communicate those, not its entire data. 
		
		\begin{figure}[htpb]
			\centering
			\includegraphics[width=0.6\linewidth]{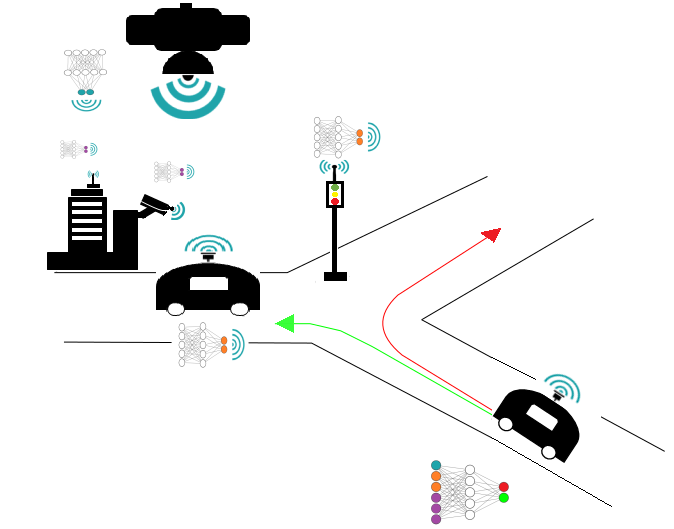}
			\caption{Fusion of inference from on-board and external sensors for automatic vehicle navigation.}
			\label{fig-autonomous-driving}
		\end{figure}
	\end{example}
	
	\begin{example}{(Public Health)}
		One of the early applications of machine learning is in the area of medical imaging and public health. In this context, various institutions can hold different modalities of patient data in the form of electronic health records, pathology test results, radiology, and other sensitive imaging data such as genetic markers for disease. Correct diagnosis may be contingent on being able to using all relevant data from all institutions. However, these institutions may not be authorized to share their raw data. Thus, it is desired to train distributively machine learning models without sharing the patient's raw data in order to prevent illegal, un-ethic or un-authorized usage of it~\cite{national2003nih}. Local hospitals or tele-health screening centers seldom acquire enough diagnostic images on their own; and collaborative distributed learning in this setting would enable each individual center to contribute data to an aggregate model without sharing any raw data. 
	\end{example}

	\subsection{Contributions}
	
	In this paper, we study the aforementioned network inference problem in which the network is modeled as a weighted acyclic graph and inference about a random variable is performed on the basis of summary information obtained from possibly correlated variables at a subset of the nodes. Following an information-theoretic approach in which we measure discrepancies between true values and their estimated fits using average logarithmic loss, we first develop a bound on the best achievable accuracy given the network communication constraints. Then, considering a supervised setting in which nodes are equipped with NNs and their mappings need to be learned from distributively available training data-sets, we propose a distributed learning and inference architecture; and we show that it can be optimized using a distributed version of the well known stochastic gradient descent (SGD) algorithm that we develop here. The resulting distributed architecture and algorithm, which we herein name ``in-network (INL) learning", generalize those introduced in~\cite{D-IB-discrete_gaussian} (see also~\cite{aguerri2019distributed,IB-problems}) for a specific case, multiaccess type, network topology. We investigate in more detail what the various nodes need to exchange during both the training and inference phases, as well as associated requirements in bandwidth. Finally, we provide a comparative study with (an adaptation of) FL and the SL of~\cite{gupta2018distributed} and experiments that illustrate our results.
	
	\subsection{Outline and Notation}
	
	In Section~\ref{prob_formulation} we describe the studied network inference problem formally. In Section~\ref{propsol} we present our in-network inference architecture, as well a distributed algorithm to training it distributively. Section~\ref{section:exp} contains a comparative study with FL and SL in terms of bandwidth requirements; as well as some experimental results.
	
	Throughout the paper the following notation will be used. Upper case letters denote random variables,e.g. $X$; lower case letters denote realizations of random variables, e.g $x$, and calligraphic letters denote sets, e.g., $\mathcal{X}$. The cardinality of a set is denoted by $|\mathcal{X}|$. For a random variable $X$ with probability mass function $P_X$, the shorthand $p(x)=P_X(x), x\in X$ is used. Boldface letters denote matrices or vectors, e.g., $\bm{X}$ or $\bm{x}$.
	For random variables $(X_1,X_2,...)$ and a set of integers $\mathcal{K} \subseteq \mathrm{N}$, the notation $X_{\mathcal{K}}$ designates the vector of random variables with indices in the set $\mc K$, i.e., $X_{\mathcal{K}} \triangleq \{X_k:k\in \mathcal{K}\}$. If $\mathcal{K}=\emptyset$ then $X_\mathcal{K}=\emptyset$.  Also, for zero-mean random vectors $\bm{x}$ and $\bm{y}$, the quantities $\bm{\sum_x}$, $\bm{\sum_{x,y}}$ and $\bm{\sum_{x|y}}$ denote, respectively, the covariance matrix of the vector $\bm{x}$, the covariance matrix of vector $(\bm{x},\bm{y})$ and the conditional covariance of $\bm{x}$ given $\bm{y}$. Finally, for two probability measures $P_X$ and $Q_X$ over the same alphabet $\mathcal{X}$, the relative entropy or Kullback-Leibler divergence is denoted as $D_{KL}(P_X || Q_X)$. That is, if $P_X$ is absolutely continuous with respect to $Q_X$, then $D_{KL}(P_X || Q_X)=\mathbb{E}_{P_X}[\log(P_X(X)/Q_X(X))]$, otherwise $D_{KL}(P_X || Q_X)=\infty$.

	\section{Network Inference: Problem Formulation}~\label{prob_formulation}
	
	\begin{figure}[!htpb]	
		\centering
		\includegraphics[width=1\linewidth]{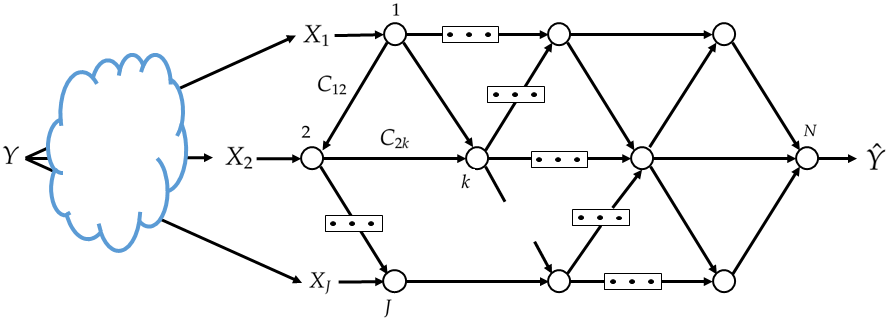}
		\caption{Studied network inference model.}
		\label{fig-network-model}
	\end{figure}
	
	Consider an $N$ node distributed network. Of these $N$ nodes, $J \geq 1$ nodes possess or can acquire data that is relevant for inference on a random variable (r.v.) of interest $Y$, with alphabet $\mc Y$. Let $\mc J =\{1, \hdots,J\}$ denote the set of such nodes, with node $j\in \mc J$ observing samples from the random variable $X_j$, with alphabet $\mc X_j$. The relationship between the r.v. of interest $Y$ and the observed ones, $X_1, \dots, X_J$, is given by the joint probability mass function\footnote{For simplicity we assume that random variables are discreet, however our technique can be applied to continuous variables as well.} $P_{X_{\mc J},Y}:=P_{X_1,\dots,X_J,Y}(x_1,\dots x_J,y)$, with $(x_1,\dots, x_j) \in \mc X_1 \times \dots \times \mc X_J$ and $y \in \mc Y$. Inference on $Y$ needs to be done at some node $N$ which is connected to the nodes that possess the relevant data through a number of intermediate other nodes. It has to be performed without any sharing of raw data. The network is modeled as a weighted directed acyclic graph; and may represent, for example, a wired network or a wireless mesh network operated in time or frequency division, where the nodes may be servers, handsets, sensors, base stations or routers. We assume that the network graph is fixed and known. The edges in the graph represent point-to-point communication links that use channel coding to achieve close to error-free communication at rates below their respective capacities. For a given loss function $\ell(\cdot,\cdot)$ that measures discrepancies between true values of $Y$ and their estimated fits, what is the best precision for the estimation of $Y$ ? Clearly, discarding any of the relevant data $X_j$ can only lead to a reduced precision. Thus, intuitively features that collectively maximize information about $Y$ need to be extracted distributively by the nodes from the set $\mc J$, without explicit coordination between them; and they then need to propagate and combine appropriately at the node $N$. How should that be performed optimally without sharing raw data ? In particular, how should each node process information from the incoming edges (if any) and what should it transmit on every one of its outgoing edges ? Furthermore, how should the information be fused optimally at Node $N$ ?
	
	\noindent More formally, we model an $N$-node network by a directed acyclic graph $\mathcal{G=(N,E, C)}$, where $\mc N=[1:N]$ is the set of nodes, $\mc E \subset \mc N \times \mc N$ is the set of edges and $\mc C=\{C_{jk}\: :\: (j,k) \in \mc E\}$ is the set of edge weights. Each node represents a device and each edge represents a noiseless communication link with capacity $C_{jk}$. The processing at the nodes of the set $\mc J$ is such that each of them assigns an index $m_{jl} \in [1,M_{jl}]$ to each $x_j \in \mc X_j$ and each received index tuple $(m_{ij}\: :\: (i,j) \in \mc E)$, for each edge $(j,l) \in \mc E$. Specifically, let for $j \in \mc J$ and $l$ such that $(j,l) \in \mc E$, the set $\mc M_{jl} = [1:M_{jl}]$. The encoding function at node $j$ is 
	\begin{equation}
		\omega_j: \mathcal{X}_j \times \left\{\: \varprod_{i\: :\: (i,j) \: \in \: \mc E} \mc M_{ij}\: \right\} \longrightarrow \varprod_{l\: :\: (j,l) \: \in \: \mc E} \mc M_{jl},
		\label{encoding-function-sender-node-j}
	\end{equation}
	where $\varprod$ designates the Cartesian product of sets. Similarly, for $k \in [1:N-1]/{\mc J}$, node $k$ assigns an index $m_{kl} \in [1,M_{kl}]$ to each index tuple $(m_{ik}\: :\: (i,k) \in \mc E)$ for each edge $(k,l) \in \mc E$. That is,
	\begin{equation}
		\omega_k: \varprod_{i\: :\: (i,k) \: \in \: \mc E} \mc M_{ik} \longrightarrow \varprod_{l\: :\: (k,l) \: \in \: \mc E} \mc M_{kl}.
		\label{encoding-function-relay-node-k}
	\end{equation}
	\noindent The range of the encoding functions $\{\omega_i\}$ are restricted in size, as
	\begin{equation}
		\log |\mc M_{ij}| \leq C_{ij} \quad \forall i \: \in \: [1,N-1] \quad \text{and} \quad \forall j\: :\: (i,j) \in \mc E.
		\label{constraint-fronthaul-capacity-links}
	\end{equation}
	Node $N$ needs to infer on the random variable $Y \in \mc Y$ using all incoming messages, i.e.,
	\begin{equation}
		\psi: \varprod_{i\: :\: (i,N) \: \in \: \mc E} \mc M_{iN} \longrightarrow \hat{\mc Y}.
		\label{decoding-function-final-decision-node}
	\end{equation}
	In this paper, we choose the reconstruction set $\hat{\mc Y}$ to be the set of distributions on $\mc Y$, i.e., $\hat{\mc Y} = \mc P(\mc Y)$; and we measure discrepancies between true values of $Y \in \mc Y$ and their estimated fits in terms of average logarithmic loss, i.e., for $(y, \hat{P}) \in \mc Y \times \mc P(\mc Y)$
	\begin{equation}
		d(y, \hat{P}) = \log \frac{1}{\hat{P}(y)}.
		\label{definition-log-loss-distorsion-measure}
	\end{equation}
	As such, the performance of a distributed inference scheme $\left( (\omega_j)_{j \in \mc J}, (\omega_k)_{k \in [1,N-1] \\ / \mc J}, \psi \right)$ for which~\eqref{constraint-fronthaul-capacity-links} is fulfilled is given by its achievable \textit{relevance} given by
	\begin{equation}
		\Delta = H(Y) - \mathbb{E}\left[d(Y,\hat{Y})\right],
		\label{relevance-measure-at-end-decision-node}
	\end{equation}
	which, for a discrete set $\mc Y$, is directly related to the error of misclassifying the variable $Y \in \mc Y$.
	
	\noindent In practice, in a supervised setting, the mappings given by~\eqref{encoding-function-sender-node-j},~\eqref{encoding-function-relay-node-k} and~\eqref{decoding-function-final-decision-node} need to be learned from a set of training data samples $\{(x_{1,i}, \hdots, x_{J,i},y_i)\}_{i=1}^n$. The data is distributed such that the samples $\dv x_j :=(x_{j,1}, \hdots, x_{j,n})$ are available at node $j$ for $j \in \mc J$ and the desired predictions $\dv y :=(y_1, \hdots, y_n)$ are available at the end decision node $N$. We parametrize the possibly stochastic mappings~\eqref{encoding-function-sender-node-j},~\eqref{encoding-function-relay-node-k} and~\eqref{decoding-function-final-decision-node} using NNs. This is depicted in Figure~\ref{fig-in-network-learning-and-inference-general-model}. We denote the parameters of the NNs that parameterize the encoding function at each node $i\in [1:(N-1)]$ with $\boldsymbol{\theta}_i$ and the parameters of the NN that parameterizes the decoding function at node $N$ with $\boldsymbol{\phi}$. Let $\boldsymbol{\theta}=[\boldsymbol{\theta}_1, \dots, \boldsymbol{\theta}_{N-1}]$, we aim to find the parameters $\boldsymbol{\theta}, \boldsymbol{\phi}$ that maximize the relevance of the network, given the network constraints of \eqref{constraint-fronthaul-capacity-links}. Given that the actual distribution is unknown and we only have access to a dataset, the loss function needs to strike a balance between its performance on the dataset, given by empirical estimate of the relevance, and the network’s ability to perform well on samples outside the dataset.

	The NNs at the various nodes are arbitrary and can be chosen independently – for instance, they need \textit{not} be identical as in FL. It is only required that the following mild condition which, as will become clearer from what follows, facilitates the back-propagation be met. Specifically, for every $j \in \mc J$ and $x_{j}\in \mc X_j$\footnote{We assume all the elements of $\mc X_j$ have the same dimension.} it holds that
	\begin{align}
		&\text{Size of first layer of NN (j)}= \nonumber\\
		&\text{Dimension}\: ( x_j) + \sum_{i\: :\: (i,j) \: \in \: \mc E}  (\text{Size of last layer of NN $(i)$}).
		\label{condition1-concatenation-activations-vectors}
	\end{align}
	Similarly, for $k \in [1:N]/{\mc J}$ we have
	\begin{align}
		&\text{Size of first layer of NN (k)} = \nonumber\\
		&\:\:\sum_{i\: :\: (i,k) \: \in \: \mc E}  (\text{Size of last layer of NN $(i)$}).
		\label{condition2-concatenation-activations-vectors}
	\end{align}
	\begin{remark}
		Conditions \eqref{condition1-concatenation-activations-vectors} and \eqref{condition2-concatenation-activations-vectors} were imposed only for the sake of ease of implementation of the training algorithm; the techniques present in this paper, including optimal trade-offs between relevance and complexity for the given topology, the associated loss function, the variational lower bound, how to parameterize it using NNs and so on, do not require \eqref{condition1-concatenation-activations-vectors} and \eqref{condition2-concatenation-activations-vectors} to hold.
	\end{remark}
	\begin{figure}[!htpb]
		\begin{subfigure}[b]{1\linewidth}
			\centering
			\includegraphics[width=0.9\linewidth]{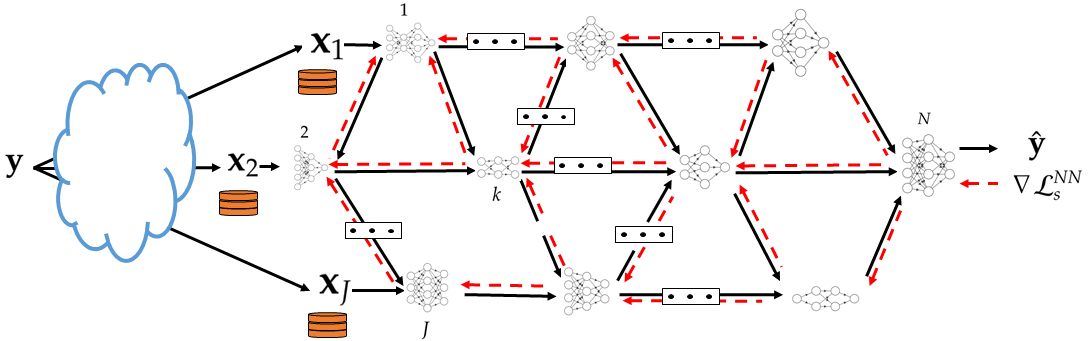}
			\caption{Training phase}
			\label{fig:prb_setting_train}
		\end{subfigure}
		\begin{subfigure}[b]{1\linewidth}
			\centering
			\includegraphics[width=0.9\linewidth]{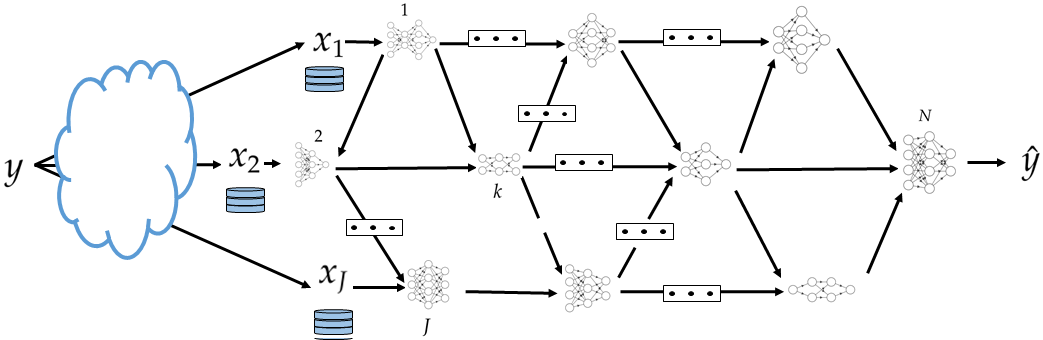}
			\caption{Inference phase}
			\label{fig:prb_setting_inf}
		\end{subfigure}
		\caption{In-network learning and inference using neural networks}
		\label{fig-in-network-learning-and-inference-general-model}
	\end{figure}

	
	\section{Proposed Solution: In-Network Learning and Inference}\label{propsol}
	
	For convenience, we first consider a specific setting of the model of network inference problem of Figure~\ref{fig-in-network-learning-and-inference-general-model} in which $J=N-1$ and all the nodes that observe data are only connected to the end decision node, but not among them.

	\subsection{A Specific Model: Fusing of Inference}~\label{wohops}
	
	In this case, a possible suitable loss function was shown by~\cite{aguerri2019distributed} to be:
	\begin{align}
		&\mc L^{\text{NN}}_s(n) = \frac{1}{n} \sum_{i=1}^n  \log Q_{\phi_{\mc J}}(y_i|u_{1,i}, \hdots, u_{J,i})\nonumber\\ &\quad +\frac{s}{n} \sum_{i=1}^n  \sum_{j=1}^J \left( \log Q_{\phi_j}(y_i|u_{j,i}) - \log \left(\frac{P_{\theta_j}(u_{j,i}|x_{j,i})}{Q_{\varphi_j}(u_{j,i})}\right) \right),
		\label{loss function}
	\end{align}
	where $s$ is a Lagrange parameter and for $j \in \mc J$ the distributions $P_{\boldsymbol{\theta_j}}(u_j|x_j)$, $Q_{\boldsymbol{\phi_j}} (y|u_j)$, $Q_{\boldsymbol{\phi_{\mc J}}}(y|u_{\mc J})$ are variational ones whose parameters are determined by the chosen NNs using the re-parametrization trick of~\cite{kingma2013auto}; and $Q_{\boldsymbol{\varphi_j}}(u_{j})$ are priors known to the encoders. For example, denoting by $f_{\theta_j}$ the NN used at node $j \in \mc J$ whose (weight and bias) parameters are given by $\boldsymbol{\theta}_j$, for regression problems the conditional distribution $P_{\boldsymbol{\theta_j}}(u_j|x_j)$ can be chosen to be multivariate Gaussian, i.e., $P_{\boldsymbol{\theta_j}}(u_j|x_j) = \mc {N} (u_j; \boldsymbol{\mu}_j^{\theta}, \dv \Sigma_j^{\theta})$. For discrete data, concrete variables (i.e., Gumbel-Softmax) can be used instead. 
	
	
	\noindent The rationale behind the choice of loss function~\eqref{loss function} is that in the regime of large $n$, if the encoders and decoder are not restricted to use NNs under some conditions~\footnote{The optimality is proved therein under the assumption that for every subset $\mc S \subseteq \mc J$ it holds that $X_{\mc S} \mkv Y \mkv X_{{\mc S}^c}$. The RHS of~\eqref{optimal loss function} is achievable for arbitrary distributions, however, regardless of such an assumption.} the optimal stochastic mappings $P_{U_j|X_j}$, $P_U$, $P_{Y|U_j}$ and $P_{Y|U_{\mc J}}$ are found by marginalizing the joint distribution that maximizes the following Lagrange cost function~\cite[Proposition 2]{aguerri2019distributed}
	\begin{equation}
		\mathcal L^{\mathrm{optimal}}_s = - H(Y|U_{\mc J}) - s \sum_{j=1}^J \Big[H(Y|U_j)+I(U_j;X_j)\Big].
		\label{optimal loss function}
	\end{equation}
	where the maximization is over all joint distributions of the form $P_{Y}\prod_{j=1}^JP_{X_j|Y}\prod_{j=1}^J P_{U_j|X_j}$.

	\subsubsection{Training Phase}
	During the forward pass, every node $j \in \mc J$ processes mini-batches of size, say, $b_j$ of its training data-set $\dv x_j$. Node $j \in \mc J$ then sends a vector whose elements are the activation values of the last layer of (NN $j$). Due to~\eqref{condition2-concatenation-activations-vectors} the activation vectors are concatenated vertically at the input layer of NN (J+1). The forward pass continues on the NN (J+1) until the last layer of the latter. 
	The parameters of NN (J+1) are updated using standard backpropgation. Specifically, let $L_{J+1}$ denote the index of the last layer of NN $(J+1)$. Also, let, 
	$\bm{w}_{J+1}^{[l]}$, $\bm{b}_{J+1}^{[l]}$ and $\bm{a}_{J+1}^{[l]}$ denote respectively the weights, biases and activation values at layer $l \in [2:L_{J+1}]$ for the NN $(J+1)$; and $\sigma$ is the activation function. Node $(J+1)$ computes the error vectors
	\begin{subequations}
		\begin{align}
			\label{back_prop_out_layer}
			& \boldsymbol{\delta}_{J+1}^{[L_{J+1}]} =\nabla_{\bm a_{J+1}^{[L_{J+1}]}}\mc L^{NN}_s(b) \odot \sigma'(\bm{w}_{J+1}^{[L_{J+1}]}\bm{a}_{J+1}^{[L_{(J+1)}-1]}+\bm{b}_{J+1}^{[L_{J+1}]}) \\
			& \boldsymbol{\delta}_{J+1}^{[l]} =[(\bm{w}_{J+1}^{[l+1]})^{T} \boldsymbol{\delta}_{J+1}^{[l+1]}]\odot \sigma'(\bm{w}_{J+1}^{[l]}\bm{a}_{J+1}^{[l-1]}+\bm{b}_{J+1}^{[l]})\:\:\: \forall \:\: l \in [2,L_{J+1}-1],
			\label{back_prop_err}\\
			& \boldsymbol{\delta}_{J+1}^{[1]} =[(\bm{w}_{J+1}^{[2]})^{T} \boldsymbol{\delta}_{J+1}^{[2]}]\label{back_prop_err_first_layer}
		\end{align} 
		\label{equations-backpropagation}
	\end{subequations}
	and then updates its weight- and bias parameters as
	\begin{subequations}
		\begin{align}
			\bm{w}^{[l]}_{J+1} &\rightarrow \bm{w}_{J+1}^{[l]}-\eta \boldsymbol{\delta}_{J+1}^{[l]}(\bm{a}_{J+1}^{[l-1]})^T,\label{back_prop_bias}\\
			\bm{b}^{[l]}_{J+1}&\rightarrow \bm{b}_{J+1}^{[l]}-\eta \boldsymbol{\delta}_{J+1}^{[l]},\label{back_prop_weight}
		\end{align}
		\label{equations-parameters-update}
	\end{subequations}
	where $\eta$ designates the learning parameter~\footnote{For simplicity $\eta$ and $\sigma$ are assumed here to be identical for all NNs.}.
	
	\begin{figure}[!hbpt]
		\centering
		\includegraphics[width=0.8\linewidth]{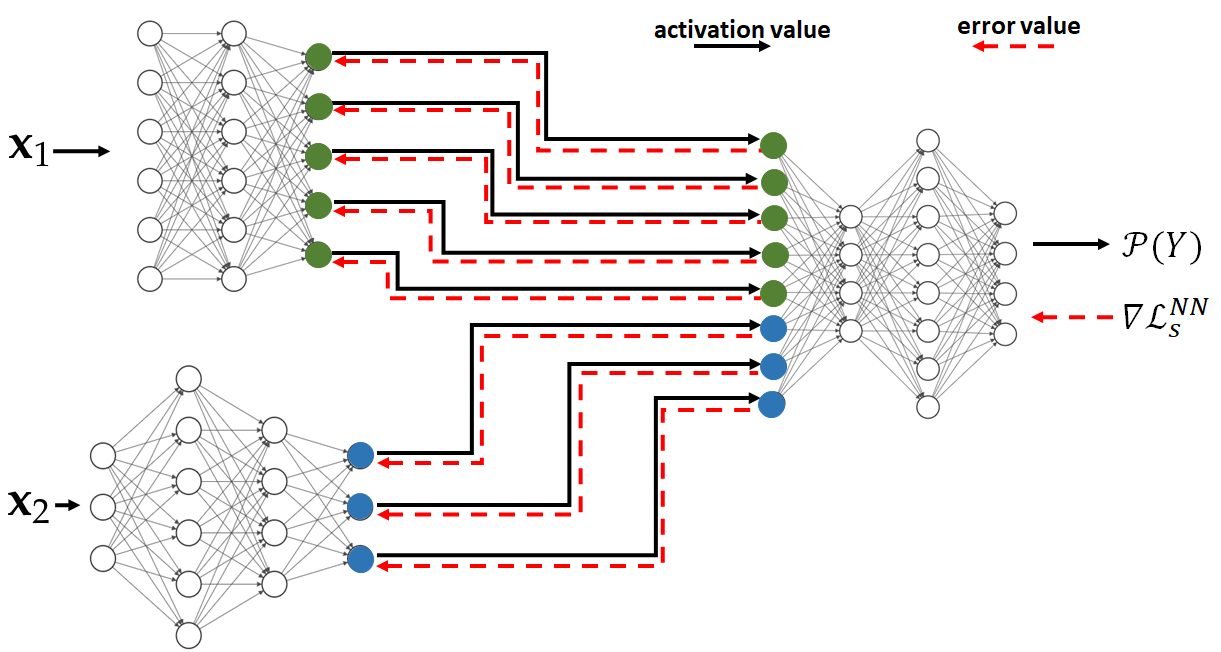}
		\caption{Forward and Backward passes for an example in-network learning with $J=2$.}
		\label{fig_2_node_mapping}
		\vspace{-0.2cm}
	\end{figure}
	
	\begin{remark}
		It is important to note that for the computation of the RHS of~\eqref{back_prop_out_layer} node $(J+1)$, which knows $Q_{\phi_{\mc J}}(y_i|u_{1,i},\hdots,u_{J,i})$ and $Q_{\phi_j}(y_i|u_{j,i})$ for all $i \in [1:n]$ and all $j \in \mc J$, only the derivative of $\mc L^{\text{NN}}_{s}(n)$ w.r.t. the activation vector $\dv a^{L_{J+1}}_{J+1}$ is required. For instance, node $(J+1)$ does not need to know any of the conditional variationals $P_{\boldsymbol{\theta_j}}(u_j|x_j)$ or the priors $Q_{\boldsymbol{\varphi_j}}(u_j)$.
	\end{remark}
	The backward propagation of the error vector from node $(J+1)$ to the nodes $j$, $j=1,\hdots,J$, is as follows. Node $(J+1)$ splits horizontally the error vector of its input layer into $J$ sub-vectors with sub-error vector $j$ having the same size as the dimension of the last layer of NN $j$ [recall~\eqref{condition2-concatenation-activations-vectors} and that the activation vectors are concatenated vertically during the forward pass]. See Figure~\ref{fig_2_node_mapping}. The backward propagation then continues on each of the $J$ input NNs simultaneously, each of them essentially applying operations similar to~\eqref{equations-backpropagation} and~\eqref{equations-parameters-update}. 
	
	\begin{remark}
		Let $\boldsymbol{\delta}_{J+1}^{[1]}(j)$ denote the sub-error vector sent back from node $(J+1)$ to node $j \in \mc J$. It is easy to see that, for every $j \in \mc J$, 
		\begin{equation}
			\nabla_{\bm a^{L_j}_{j}}\mc L^{NN}_s(b_j)=\boldsymbol{\delta}_{J+1}^{[1]}(j)-s\nabla_{\bm a^{L_j}_{j}}\left(\sum_{i=1}^b \log \left(\frac{P_{\boldsymbol{\theta_j}}(u_{j,i}|x_{j,i})}{Q_{\boldsymbol{\varphi_j}}(u_{j,i})}\right)\right);
		\end{equation}
		and this explains why node $j \in \mc J$ needs only the part $\boldsymbol{\delta}_{J+1}^{[1]}(j)$, not the entire error vector at node $(J+1)$.
	\end{remark}
	\begin{figure}[!htpb]
		\begin{subfigure}[b]{1\linewidth}
			\centering
			\includegraphics[width=0.9\linewidth]{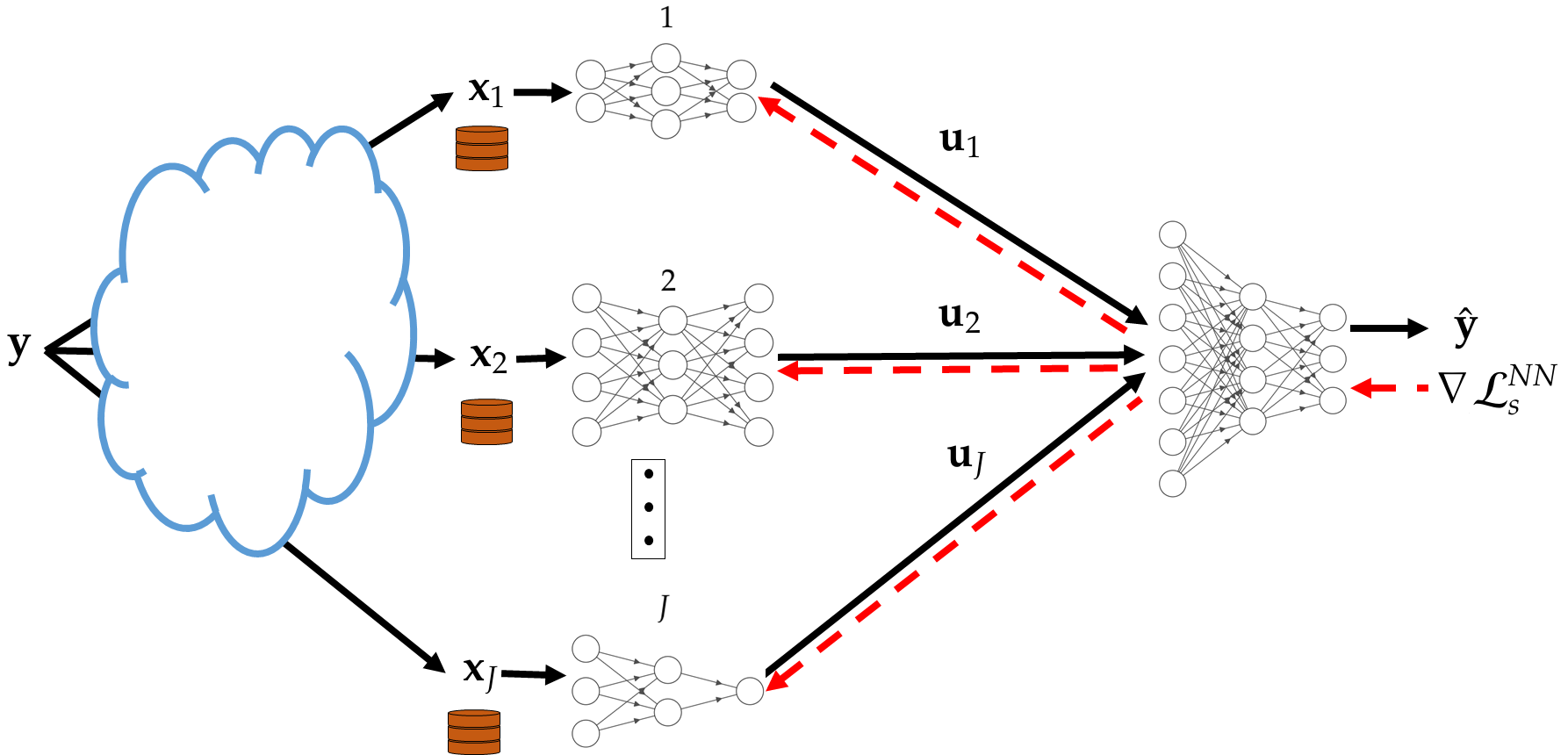}
			\caption{Training phase}
			\label{fig:prb_setting_train}
		\end{subfigure}
		\begin{subfigure}[b]{1\linewidth}
			\centering
			\includegraphics[width=0.9\linewidth]{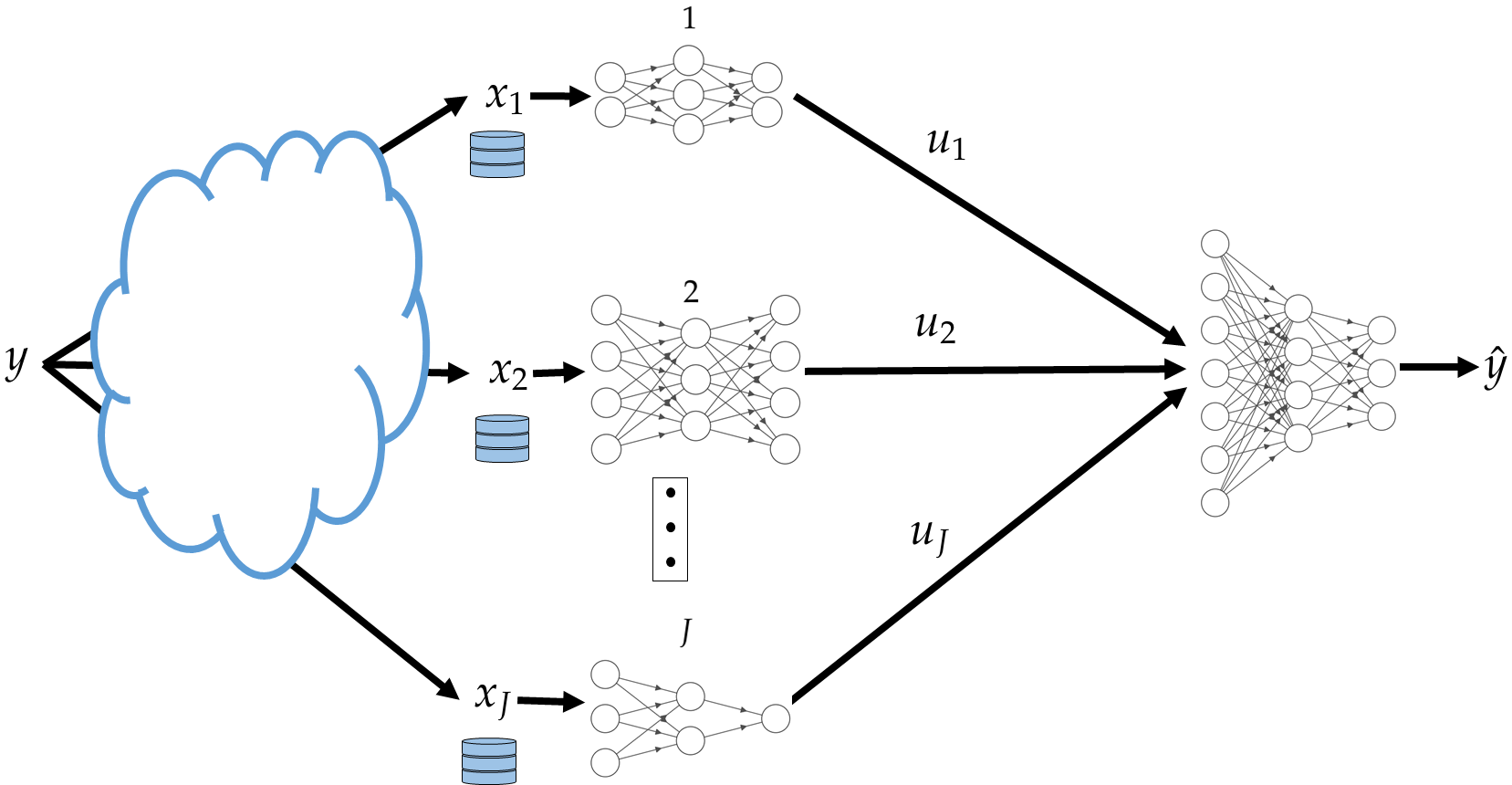}
			\caption{Inference phase}
			\label{fig:prb_setting_inf}
		\end{subfigure}
		\caption{In-network learning for the network model for the case without hops}
		\label{fig:multi_view_learning}
	\end{figure}
	\subsubsection{Inference Phase}
	
	During this phase node $j$ observes a new sample $x_j$. It uses its NN to output an encoded value $u_j$ which it sends to the decoder. After collecting $(u_1, \cdots, u_J)$ from all input NNs, node $(J+1)$ uses its NN to output an estimate of $Y$ in the form of soft output $Q_{\phi_{\mc J}}(Y|u_1,\hdots, u_J)$. The procedure is depicted in Figure~\ref{fig:prb_setting_inf}.
	
	\begin{remark}
		A suitable practical implementation in wireless settings can be obtained using Orthogonal Frequency Division Multiplexing (OFDM). That is, the $J$ input nodes are allocated non-overlapping bandwidth segments and the output layers of the corresponding NNs are chosen accordingly. The encoding of the activation values can be done, e.g., using entropy type coding~\cite{FHH-L21}. 
	\end{remark}
	
	\subsection{General Model: Fusion and Propagation of Inference}~\label{whops}
	
	Consider now the general network inference model of Figure~\ref{fig-network-model}. Part of the difficulty of this problem is in finding a suitable loss function and that can be optimized distributively via NNs that only have access to local data-sets each. The next theorem provides a bound on the relevance achievable (under some assumptions~\footnote{The inference problem is a one-shot problem. The result of Theorem~\ref{thorem_achive} is asymptotic in the size of the training data-sets. One-shot results for this problem can be obtained, e.g., along the approach of~\cite{CG18}. }) for an arbitrary network topology $(\mc E,\mc N)$. For convenience, we define for $\mc S \subseteq [1,\hdots,N-1]$ and non-negative $(C_{ij} \: : \: (i,j) \in \mc E)$ the quantity
	\begin{equation}
		C(\mc S) =\sum_{(i,j) \: : \:  i \in \mc S, j \in \mc S^c} C_{ij}.
		\label{definition-cut-set}
	\end{equation}

	\begin{theorem}
		For the network inference model of Figure~\ref{fig-network-model}, in the regime of large data-sets the following relevance is achievable, 
		\begin{equation}
			\Delta = \max I(U_1,\hdots,U_J;Y)
			\label{delta_bound}
		\end{equation}
		where the maximization is over joint measures of the form 
		\begin{equation}
			P_{Q}P_{X_1,\hdots,X_J,Y} \prod_{j=1}^J P_{U_j|X_j,Q}
			\label{joint-measure-statement-theorem1}
		\end{equation}
		for which there exist non-negative $R_1,\hdots,R_J$ that satisfy
		\begin{align*}
			& \sum_{j \in \mc S} R_j \geq I(U_{\mc S};X_{\mc S}|U_{\mc S ^c},Q), \quad \text{for all}\quad \mc S \subseteq \mc J \\
			& \sum_{j \in \mc S \cap \mc J} R_j \leq C(\mc S) \quad \text{for all}\quad \mc S \subseteq [1:N-1] \quad  \text{with} \quad \mc S \cap \mc J \neq \emptyset.
		\end{align*}
		\label{thorem_achive}
	\end{theorem}

%
%
%
	\begin{proof}
		The proof of Theorem~\ref{thorem_achive} appears in Appendix \ref{appendix-proof-capacity-relavance-region}. An outline is as follows. The result is achieved using a separate compression-transmission-estimation scheme in which the observations $(\dv x_1,\hdots, \dv x_J)$ are first compressed distributively using Berger-Tung coding~\cite{berger_tung_coding} into representations $(\dv u_1,\hdots,\dv u_J)$; and, then, the bin indices are transmitted as independent messages over the network $\mc G$ using linear-network coding~\cite[Section 15.5]{elgamalkim2011}. The decision node $N$ first recovers the representation codewords $(\dv u_1,\hdots,\dv u_J)$; and, then, produces an estimate of the label $\dv y$. The scheme is illustrated in Figure~\ref{fig-illustration-proof-theorem1}.
	\end{proof}
	
	\begin{figure}[!htpb]
		\centering
		\begin{subfigure}[b]{1\linewidth}
			\centering
			\includegraphics[width=0.8\linewidth]{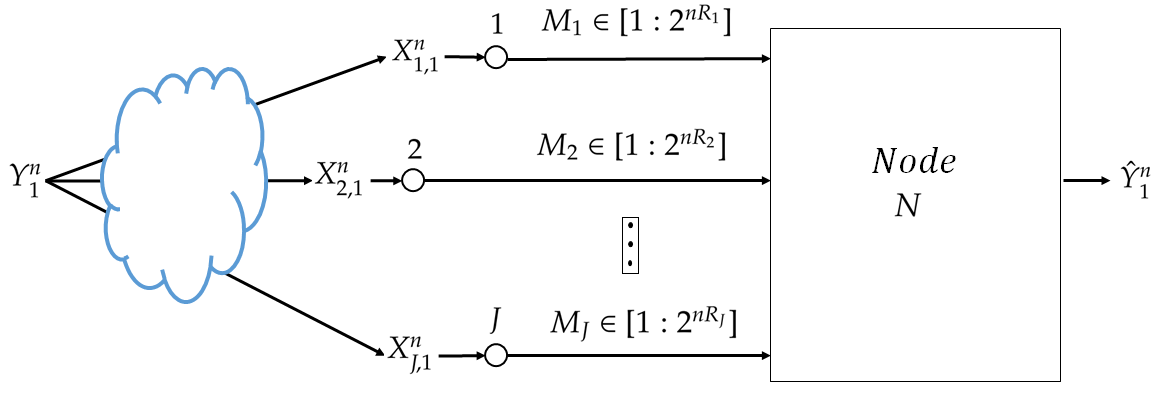}
			\caption{Compression using Berger-Tung coding\newline }
			\label{fig-source_coding}
		\end{subfigure}
		\begin{subfigure}[b]{1\linewidth}
			\centering
			\includegraphics[width=0.8\linewidth]{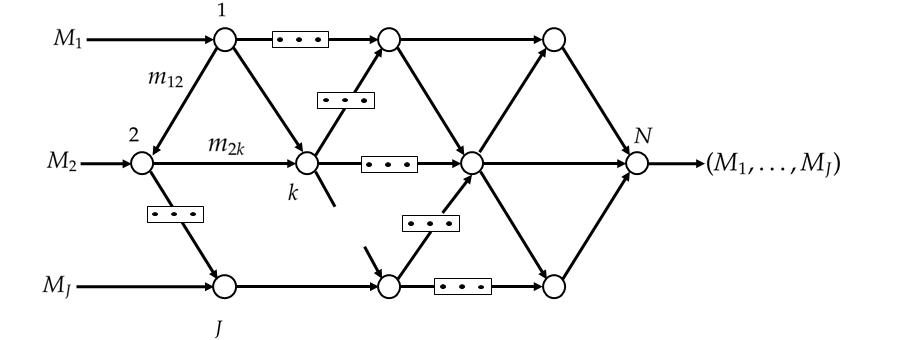}
			\caption{Transmission of the bin indices using linear coding}
			\label{fig-channel_coding}
		\end{subfigure}
		\caption{Block diagram of the separate compression-transmission-estimation scheme of Theorem~\ref{thorem_achive}}
		\label{fig-illustration-proof-theorem1}
	\end{figure}
	
	Part of the utility of the loss function of Theorem~\ref{thorem_achive} is in that it accounts explicitly for the topology of the network for inference fusion and propagation. Also, although as seen from its proof the setting of Theorem~\ref{thorem_achive} assumes knowledge of the joint distribution of the tuple $(X_1,\hdots,X_J,Y)$, the result can be used to train, distributively, NNs from a set of available date-sets. To do so, we first derive a Lagrangian function, from Theorem~\ref{thorem_achive}, which can be used as an objective function to find the desired set of encoders and decoder. Afterwards, we use a variational approximation to avoid the computation of marginal distributions, which can be costly in practice. Finally, we parameterize the distributions suing NNs. For a given network topology in essence, the approach generalizes that of Section~\ref{wohops} to more general networks that involve hops. 
	For simplicity, in what follows, this is illustrated for the example architecture of Figure~\ref{fig:netrelaygraphsimp}. While the example is simple, it showcases the important aspect of any such topology, the fusion of the data at an intermediary nodes, i.e., a hop.

	\begin{figure}[htpb]
		\centering
		\includegraphics[width=1\linewidth]{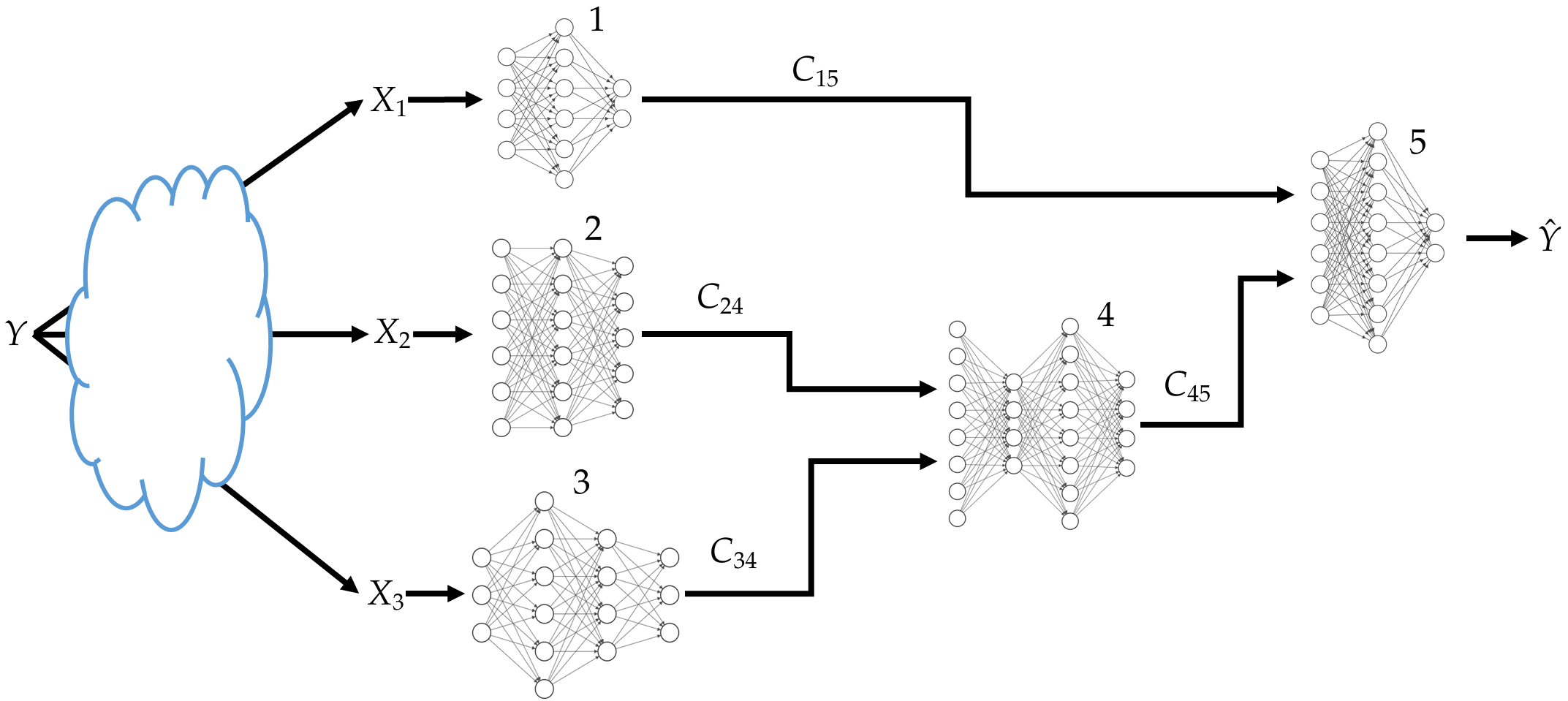}
		\caption{An example in-network learning with inference fusion and propogation}
		\label{fig:netrelaygraphsimp}
	\end{figure}
	
	\noindent Setting $\mc N=\{1,2,3,4,5\}$ and $\mc E=\{(3,4),(2,4),(4,5),(1,5)\}$ in Theorem~\ref{thorem_achive}, we get that 
	\begin{equation}
		\Delta = \max I(U_1,U_2,U_3;Y)
		\label{delta-bound-example-network-with-hops}
	\end{equation}
	where the maximization is over joint measures of the form 
	\begin{equation}
		P_{Q}P_{X_1,X_2,X_3,Y} P_{U_1|X_1,Q} P_{U_2|X_2,Q} P_{U_3|X_3,Q}
	\end{equation}
	for which the following holds for some $R_1 \geq 0$, $R_2 \geq 2$ and $R_3 \geq 0$:
	\begin{subequations}
		\begin{align}
			& C_{15} \geq R_1,\: C_{24} \geq R_2, \: C_{34} \geq R_3, \: C_{45} \geq R_2+R_3\\
			& R_1 \geq I(U_1;X_1|U_2,U_3,Q),\\
			& R_2  \geq I(U_2;X_2|U_1,U_3,Q),\\
			& R_3 \geq I(U_3;X_3|U_1,U_2,Q)\\
			& R_3+R_2 \geq I(X_2,X_3;U_2,U_3|U_1,Q),\\
			& R_3+R_1 \geq I(X_1,X_3;U_1,U_3|U_2,Q)\\
			& R_2+R_1 \geq I(X_1,X_2;U_1,U_2|U_3,Q),\\
			& R_2+R_1+R_3 \geq I(X_1,X_2,X_3;U_1,U_2,U_3|Q).
		\end{align}
		\label{region-5-node}
	\end{subequations}
	
	\noindent Let $C_{\text{sum}}=C_{15}+C_{24}+C_{34}+C_{45}$; consider the region of all pairs $(\Delta,C_{\text{sum}}) \in \mathbb{R}^2_{+}$ for which relevance level $\Delta$ as given by the RHS of~\eqref{delta-bound-example-network-with-hops} is achievable for some $C_{15} \geq 0$, $C_{24} \geq 0$, $C_{34} \geq 0$ and $C_{45} \geq 0$ such that $C_{\text{sum}}=C_{15}+C_{24}+C_{34}+C_{45}$. Hereafter, we denote such region as $\mc{RI}_{\text{sum}}$. Applying Fourier-Motzkin elimination on~ the region defined by~\eqref{delta-bound-example-network-with-hops} and~\eqref{region-5-node}, we get that the region $\mc{RI}_{\text{sum}}$ is given by the union of pairs $(\Delta, C_{\text{sum}}) \in \mathbb{R}^2_{+}$ for which~\footnote{The time sharing random variable is set to a constant for simplicity.}
	\begin{subequations}
		\begin{align}
			\Delta & \leq I\left(Y;U_1,U_2,U_3\right)\\
			C_{\text{sum}} & \geq I(X_1,X_2,X_3;U_1,U_2,U_3)+I(X_2,X_3;U_2,U_3|U_1)\label{eq_csum_ineq_def}
		\end{align}
		\label{relevance-complexity-region-example-network-with-hops}
	\end{subequations}
	for some measure of the form	
	\begin{equation}
		P_{Y} P_{X_1,X_2,X_3|Y} P_{U_1|X_1}P_{U_2|X_2}P_{U_3|X_3}.
		\label{distribution-relevance-complexity-region-example-network-with-hops}	
	\end{equation}
	
	\noindent The next proposition gives a useful parametrization of the region $\mc{RI}_{\text{sum}}$ as described by~\eqref{relevance-complexity-region-example-network-with-hops} and~\eqref{distribution-relevance-complexity-region-example-network-with-hops}.
	
	\begin{proposition}~\label{proposition-parametrization-region-theorem1}
		For every pair $(\Delta,C_{\text{sum}})$ that lies on the boundary of the region described by \eqref{relevance-complexity-region-example-network-with-hops} and~\eqref{distribution-relevance-complexity-region-example-network-with-hops} there exists $s\geq 0$ such that  $(\Delta,C_{sum})=(\Delta_s,C_s)$, with
		\begin{subequations}
			\begin{align}
				\Delta_s&=H(Y)+\max_{\mathbf P}\mc L_s(\mathbf P)+s C_s\\
				C_{s}&=I(X_1,X_2,X_3;U^*_1,U^*_2,U^*_3)+I(X_2,X_3;U^*_2,U^*_3|U^*_1),
			\end{align}\label{region-5-node-sum-delta_s-c_s}
		\end{subequations}
		and $\mathbf P^*$ is the set of pmfs $\mathbf P := \lbrace P_{U_1|X_1},P_{U_2|X_2}, P_{U_3|X_3}\rbrace$ that maximize the cost function
		\begin{align}
			\mc L_s(\mathbf P) &:=-H(Y|U_1,U_2,U_3)-sI(X_1,X_2,X_3;U_1,U_2,U_3)\nonumber\\
			&-sI(X_2,X_3;U_2,U_3|U_1).
			\label{loss_function_hops_5}
		\end{align}	
	\end{proposition}
	\begin{proof}
		See Appendix \ref{appendix-proof-parametrization-region-theorem1}.
	\end{proof}
	
	\noindent In accordance with the studied example network inference problem of Figure~\ref{fig:netrelaygraphsimp}, let a random variable $U_4$ be such that $U_4 \mkv (U_2,U_3) \mkv (X_1,X_2,X_3,Y,U_1)$. That is, the joint distribution factorizes as
	\begin{equation}
		P_{X_1,X_2,X_3,Y,U_1,U_2,U_3,U_4} = P_{X_1,X_2,X_3,Y} P_{U_1|X_1} P_{U_2|X_2} P_{U_3|X_3} P_{U_4|U_2,U_3}.
		\label{augmented-joint-measure-example-network-inference-with-hops}
	\end{equation}
	\noindent Let for given $s \geq 0$ and conditional $P_{U_4|U_2,U_3}$ the Lagrange term
	\begin{align}
		\mc L_s^{\text{low}}(\bm P,P_{U_4|U_2,U_3}) =&-H(Y|U_1,U_4)-sI(X_1;U_1)-2sI(X_2;U_2)\nonumber\\
		&-2s\Big[I(X_3;U_3)- I(U_2;U_1)-I(U_3;U_1,U_2)\Big].
		\label{loss_eq_low_ex}
	\end{align}
	\noindent The following lemma shows that $\mc  L_s^{\text{low}}(\bm P,P_{U_4|U_2,U_3})$ lower bounds $\mc L_s(\dv P)$ as given by~\eqref{loss_function_hops_5}.

	\begin{lemma}
		For every $s \geq 0$ and joint measure that factorizes as~\eqref{augmented-joint-measure-example-network-inference-with-hops}, we have 
		\begin{equation}
			\mc L_s(\mathbf P) \geq \mc L_s^{\text{low}}(\bm P,P_{U_4|U_2,U_3}),
		\end{equation}
		\label{lemma-low-bound}
	\end{lemma}
	\vspace{-0.4cm}
	\begin{proof}
		See Appendix \ref{appendix-proof-lower-bound}.
	\end{proof}
	
	\noindent For convenience let $\mathbf P_+ := \lbrace P_{U_1|X_1},P_{U_2|X_2},P_{U_3|X_3},P_{U_4|U_2,U_3}\rbrace$. The optimization of \eqref{loss_eq_low_ex} generally requires the computation of marginal distributions, which can be costly in practice. Hereafter we derive a variational lower bound on $\mc  L_s^{\text{low}}$ with respect to some arbitrary (variational) distributions.
	Specifically, let
	\begin{equation}
		\bm Q := \{ Q_{Y|U_1,U_4},Q_{U_3},Q_{U_2},Q_{U_1}\},
	\end{equation}
	where $Q_{Y|U_1,U_4}$ represents variational (possibly stochastic) decoders and $Q_{U_3}$, $Q_{U_2}$ and $Q_{U_1}$ represent priors. Also, let
	\begin{align}
		\mc L_s^{\text{v-low}}(\bm P_+,\bm Q):=&\mathbb{E}[\log Q_{Y|U_1,U_4}(Y|U_1,U_4)]-sD_{\mathrm{KL}}(P_{U_1|X_1}\Vert Q_{U_1})\nonumber\\
		&-2sD_{\mathrm{KL}}(P_{U_2|X_2}\Vert Q_{U_2})-2sD_{\mathrm{KL}}(P_{U_3|X_3}\Vert Q_{U_3}).
		\label{loss_prob_function_hops_5}
	\end{align}
	
	\noindent The following lemma, the proof of which is essentially similar to that of~\cite[Lemma 1]{aguerri2019distributed}, shows that for every $s \geq 0$, the cost function $\mc L_s^{\text{low}}(\bm P,P_{U_4|U_2,U_3})$ is lower-bounded by $\mc L_s^{\text{v-low}}(\bm P_+,\bm Q)$ as given by~\eqref{loss_prob_function_hops_5}. 
	
	\begin{lemma}
		For fixed $\bm P_+$, we have
		\begin{equation}
			\mc L_s^{\text{low}}(\bm P_+)\geq  \mc L_s^{\text{v-low}}(\bm P_+,\bm Q)
		\end{equation}
		for all pmfs $\bm Q$, with equality when:
		\begin{align}
			Q_{Y|U_1,U_4}&=P_{Y|U_1,U_4},\\
			Q_{U_3}&=P_{U_3|U_2,U_1},\\
			Q_{U_2}&=P_{U_2|U_1},\\
			Q_{U_1}&=P_{U_1},
		\end{align}
		where $P_{Y|U_1,U_4}$,$P_{U_3|U_2,U_1}$,$P_{U_2|U_1}$,$P_{U_1}$ are calculated using~\eqref{augmented-joint-measure-example-network-inference-with-hops}.\label{lemma-v_low_bound}
	\end{lemma}
	\begin{proof}
		See Appendix \ref{appendix-proof-variational-lower-bound}.
	\end{proof}
	
	\noindent From the above, we get that
	\begin{equation}
		\max_{\bm P_+}\mc L_s^{\text{low}}(\bm P_+) = \max_{\bm P_+} \max_{\bm Q} \mc L_s^{\text{v-low}}(\bm P_+,\bm Q).
	\end{equation}
	Since, as described in Section~\ref{prob_formulation}, the distribution of the data is not known, but only a set of samples is available $\{(x_{1,i},\hdots,x_{J,i},y_i)\}_{i=1}^n$, we restrict the optimization of \eqref{loss_prob_function_hops_5} to the family of distributions that can be parametrized by NNs. Thus, we obtain the following loss function which can be optimized empirically, in a distributed manner, using gradient based techniques,
	\begin{align}
		\mc L^{\text{NN}}_s(n)&:=\frac{1}{n} \sum_{i=1}^n \Bigg[ \log Q_{\phi_5}(y_i|u_{1,i},u_{4,i}) -s \log \left(\dfrac{P_{\theta_1}(u_{1,i}|x_{1,i})}{Q_{\varphi_1}(u_{1,i})}\right)\Bigg]\nonumber\\
		&-\dfrac{2s}{n} \sum_{i=1}^n \Bigg[ \log \left(\dfrac{P_{\theta_2}(u_{2,i}|x_{2,i})}{Q_{\varphi_2}(u_{2,i})}\right)+\log \left(\frac{P_{\theta_3}(u_{3,i}|x_{3,i})}{Q_{\varphi_3}(u_{3,i})}\right)  \Bigg],
		\label{loss function_5_node}
	\end{align}
	with $s$ stands for a Lagrange multiplier and the distributions $Q_{\phi_5},P_{\theta_4},P_{\theta_3},P_{\theta_2},P_{\theta_1}$ are variational ones whose parameters are determined by the chosen NNs using the re-parametrization trick of \cite{kingma2013auto}; and, $\{Q_{\varphi_i}\: : \:  i \in \{1,2,3\} \}$ are priors known to the encoders. The parametrization of the distributions with NNs is performed similarly to that for the setting of Section~\ref{wohops}.
	\begin{figure}[htpb]
		\centering
		\includegraphics[width=1\linewidth]{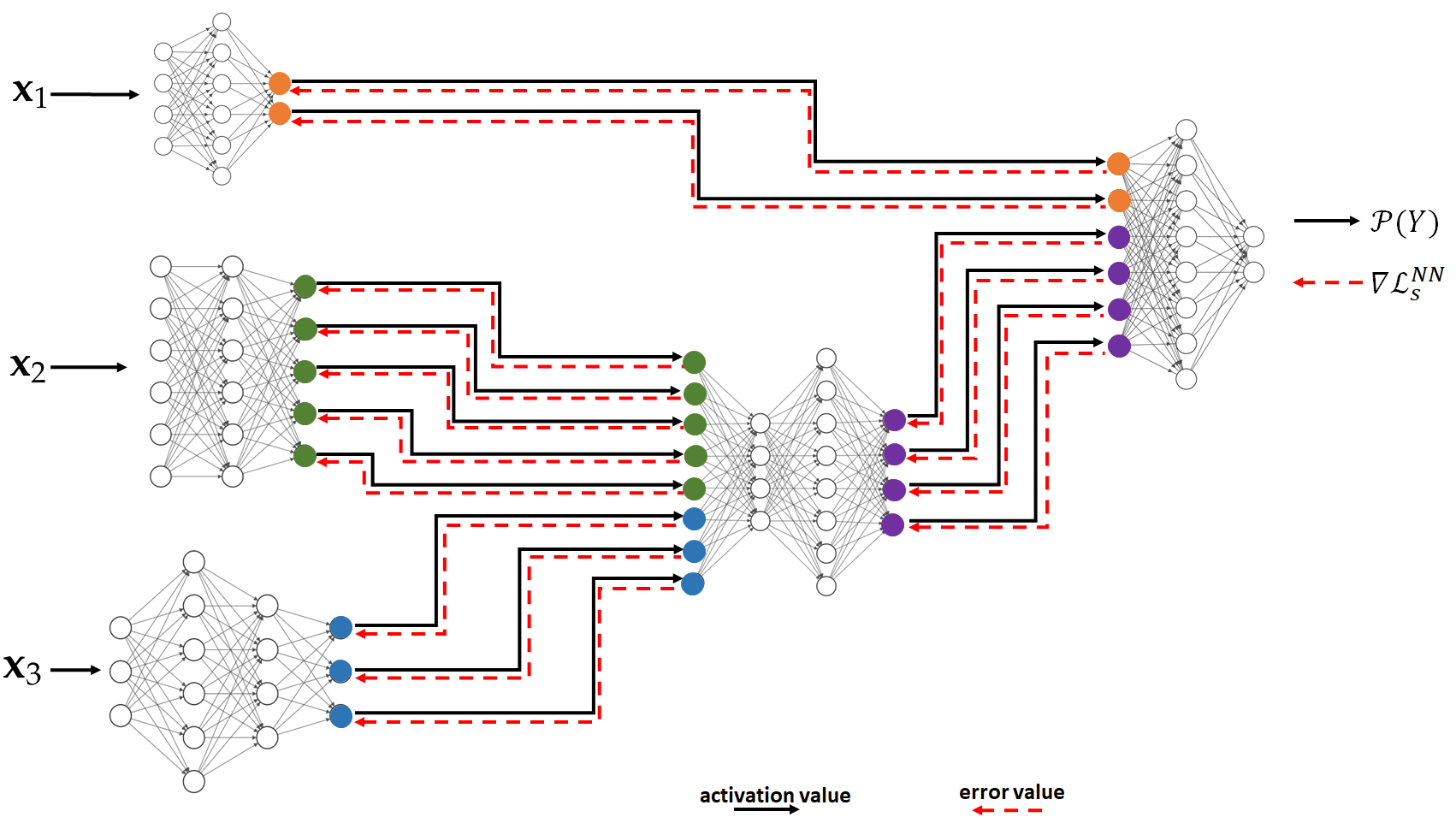}
		\caption{Forward and backward passes for the inference problem of Figure~\ref{fig:netrelaygraphsimp}}
		\label{fig-nn_fig_5_nodes}
	\end{figure}
	
	\subsubsection{Training Phase}
	During the forward pass, every node $j \in \{1,2,3\}$ processes mini-batches of size, $b_j$ of its training data set $\dv x_j$. Nodes $2$ and $3$ send their vector formed of the activation values of the last layer of their NNs to node $4$. Because the sizes of the last layers of the NNs of nodes $2$ and $3$ are chosen according to~\eqref{condition2-concatenation-activations-vectors} the sent activation vectors are concatenated vertically at the input layer of NN $4$. The forward pass continues on the NN at node $4$ until its last layer. Next, nodes $1$ and $4$ send the activation values of their last layers to node $5$. Again, as the sizes of the last layers of the NNs of nodes $1$ and $4$ satisfy~\eqref{condition2-concatenation-activations-vectors} the sent activation vectors are concatenated vertically at the input layer of NN $5$; and the forward pass continues until the last layer of NN $5$.
	
	\noindent During the backward pass, each of the NNs updates its parameters according to~\eqref{equations-backpropagation} and~\eqref{equations-parameters-update}. Node $5$ is the first to apply the back propagation procedure in order update the parameters of its NN. It applies \eqref{equations-backpropagation} and~\eqref{equations-parameters-update} sequentially, starting from its last layer.
	
	\begin{remark}
		It is important to note that, similar to the setting of Section III-A, for the computation of the RHS of~\eqref{back_prop_out_layer} for node $5$, only the derivative of $\mc L^{\text{NN}}_{s}(n)$ w.r.t. the activation vector $\dv a^{L_{5}}_{5}$ is required, which depends only on $Q_{\phi_5}(y_i|u_{1,i},u_{4,i})$. The distributions are known to node 5 given only $u_{1,i}$ and $u_{4,i}$.
	\end{remark}
	
	The error propagates back until it reaches the first layer of the NN of node $5$. Node $5$ then splits horizontally the error vector of its input layer into $2$ sub-vectors with the top sub-error vector having as size that of the last layer of the NN of node $1$ and the bottom sub-error vector having as size that of the last layer of the NN of node $4$ -- see Figure~\ref{fig-nn_fig_5_nodes}. Similarly, the two nodes $1$ and $4$ continue the backward propagation at their turns simultaneously. Node $4$ then splits horizontally the error vector of its input layer into $2$ sub-vectors with the top sub-error vector having as size that of the last layer of the NN of node $2$ and the bottom sub-error vector having as size that of the last layer of the NN of node $3$. Finally, the backward propagation continues on the NNs of nodes $2$ and $3$. The entire process continues until convergence. 
	
	\begin{remark}
		Let $\boldsymbol{\delta}_{J+1}^{[1]}(j)$ denote the sub-error vector sent back from node $(J+1)$ to node $j \in \mc J$. It is easy to see that, for every $j \in \mc J$, 
		\begin{align*}		
			\nabla_{\bm a_{4}^{[L]}}\mc L^{NN}_s(b)&=\boldsymbol{\delta}_{5}^{[1]}(4),\\
			\nabla_{\bm a_{3}^{[L]}}\mc L^{NN}_s(b)&=\boldsymbol{\delta}_{4}^{[1]}(3)-2s\nabla_{\bm a_{3}^{[L]}}\left[  \frac{1}{b} \sum_{i=1}^b \left[\log \left(\frac{P_{\theta_3}(u_{3,i}|x_{3,i})}{Q_{\varphi_3}(u_{3,i})}\right)\right]\right],\\
			\nabla_{\bm a_{2}^{[L]}}\mc L^{NN}_s(b)&=\boldsymbol{\delta}_{4}^{[1]}(2)-2s\nabla_{\bm a_{2}^{[L]}}\left[  \frac{1}{b} \sum_{i=1}^b \left[\log \left(\frac{P_{\theta_2}(u_{2,i}|x_{2,i})}{Q_{\varphi_2}(u_{2,i})}\right)\right]\right],\\
			\nabla_{\bm a_{1}^{[L]}}\mc L^{NN}_s(b)&=\boldsymbol{\delta}_{5}^{[1]}(1)-s\nabla_{\bm a_{1}^{[L]}}\left[  \frac{1}{b} \sum_{i=1}^b \left[\log \left(\frac{P_{\theta_1}(u_{1,i}|x_{1,i})}{Q_{\varphi_1}(u_{1,i})}\right)\right]\right].
		\end{align*}
		and this explains why, for back propagation, nodes $1,2,3,4$ need only part of the error vector at the node they are connected to.
	\end{remark}
	
	\subsubsection{Inference Phase}
	
	During this phase, nodes $1$, $2$ and $3$ observe (or measure) each a new sample. Let $\dv x_1$ be the sample observed by node $1$; and $\dv x_2$ and $3$ those observed by node $2$ and node $3$, respectively. Node $1$ processes $\dv x_1$ using its NN and sends an encoded value $\dv u_1$ to node $5$; and so do nodes $2$ and $3$ towards node $4$. Upon receiving $\dv u_2$ and $\dv u_3$ from nodes $2$ and $3$, node $4$ concatenates them vertically and processes the obtained vector using its NN. The output $\dv u_4$ is then sent to node $5$. The latter performs similar operations on the activation values $\dv u_1$ and $\dv u_4$; and outputs an estimate of the label $\dv y$ in the form of a soft output $Q_{\phi_5}(\dv y|\dv u_1,\dv u_4)$. 
	
	\subsection{Bandwidth requirements}\label{bandreq}
	
	In this section, we study the bandwidth requirements of our in-network learning. Let $q$ denote the size of the entire data set (each input node has a local dataset of size $\frac{q}{J}$),  $p=L_{J+1}$ the size of the input layer of NN $(J+1)$ and $s$ the size in bits of a parameter. Since as per~\eqref{condition2-concatenation-activations-vectors}, the output of the last layers of the input NNs are concatenated at the input of NN $(J+1)$ whose size is $p$, and each activation value is $s$ bits, one then needs $\dfrac{2sp}{J}$ bits for each data point -- the factor $2$ accounts for both the forward and backward passes; and, so, for an epoch our in-network learning requires $\dfrac{2pqs}{J}$ bits. 
	
	Note that the bandwidth requirement of in-network learning does not depend on the sizes of the NNs used at the various nodes, but does depend on the size of the dataset. For comparison, notice that  with FL one would require $2NJs$, where $N$ designates the number of (weight- and bias) parameters of a NN at one node. For the SL of~\cite{gupta2018distributed}, assuming for simplicity that the NNs $j=1,\hdots,J$ all have the same size $\eta N$, where $\eta \in [0,1]$, SL requires $(2pq+\eta NJ)s$ bits for an entire epoch.
	
	The bandwidth requirements of the three schemes are summarized and compared in Table~\ref{band_res} for two popular NNs architectures, VGG16 ($N=138,344,128$ parameters) and ResNet50 ($N=25,636,712$ parameters) and two example datsets, $q =50, 000$ data points and $q=500, 000$ data points. The numerical values are set as $J=500$, $p=25088$ and $\eta=0.88$ for ResNet50 and $0.11$ for VGG16. 
	
	\begin{table}[!ht]
		\centering
		
		\begin{tabular}{|c|l|l|l|}
			\hline
			& \thead{Federated\\learning} & \thead{Split\\learning} & \thead{In-network\\learning} \\ \hline
			\thead{Bandwidth\\requirement}  & $	2NJs$          &    $\left(2pq+\eta NJ\right)s$            &   \thead{$\dfrac{2pqs}{J}$}    \\ \hline
			\thead{VGG 16  \\ 50,000 data points}   &   4427 Gbits  &  324 Gbits     & 0.16 Gbits     \\ \hline
			\thead{ResNet 50  \\ 50,000 data points}	&   820 Gbits    &  441 Gbits & 0.16 Gbits     \\ \hline
			\thead{VGG 16  \\ 500,000 data points}  		&   4427 Gbits      &   1046 Gbits   & 1.6 Gbits           \\ \hline
			\thead{ResNet 50 \\ 500,000 data points} 		&   820 Gbits   &   1164 Gbits   & 1.6 Gbits            \\ \hline
		\end{tabular}
		\caption{Comparison of bandwidth requirements}
		\label{band_res}
	\end{table}

	Compared to FL and SL, INL has an advantage in that all nodes work jointly also during inference to make a prediction,not just during the training phase. As a consequence nodes only need to exchange latent representations, not model parameters, during training.
	
	\section{Experimental Results}\label{section:exp}
	
	We perform two series of experiments for which we compare the performance of our INL with those of FL and SL. The dataset used is the CIFAR-10 and there are five client nodes. In the first experiment the three techniques are implemented in such a way such that during the inference phase the same NN is used to make the predictions. In the second experiment the aim is to implement each of the techniques such that the data is spread in the same manner across the five client nodes for each of the techniques. 
	\vspace{-0.3cm}
	\subsection{Experiment 1}
	In this setup, we create five sets of noisy versions of the images of CIFAR-10. To this end, the CIFAR images are first normalized, and then corrupted by additive Gaussian noise with standard deviation set respectively to $0.4, 1, 2, 3, 4$. 
	For our INL each of the five input NNs is trained on a different noisy version of the same image. Each NN uses a variation of the VGG network of~\cite{vgg_small_data}, with the categorical cross-entropy as the loss function, L2 regularization, and Dropout and BatchNormalization layers. Node $(J+1)$ uses two dense layers. The architecture is shown in Figure~\ref{fig:nt_arh_5}. In the experiments, all five (noisy) versions of every CIFAR-10 image are processed simultaneously, each by a different NN at a distinct node, through a series of convolutional layers. The outputs are then concatenated and then passed through a series of dense layers at node $(J+1)$. 
	
	\begin{figure}[!t]
		\centering
		\includegraphics[width=0.6\linewidth]{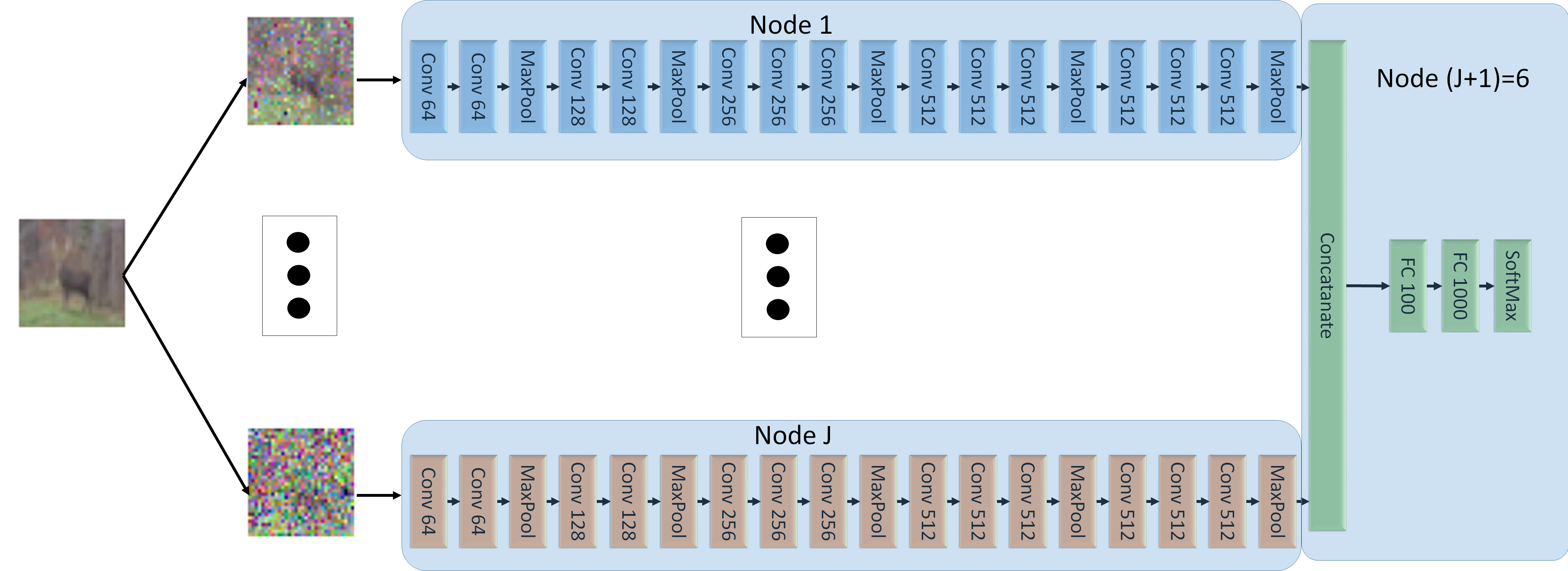}
		\caption{Network architecture. \textit{Conv} stands for a convolutional layer, \textit{Fc} stand for a fully connected layer.}
		\label{fig:nt_arh_5}
		\vspace{-0.2cm}
	\end{figure}

		\begin{figure}[!ht]
		\centering
		\begin{subfigure}{0.7\linewidth}
			\centering
			\includegraphics[width=0.8\linewidth]{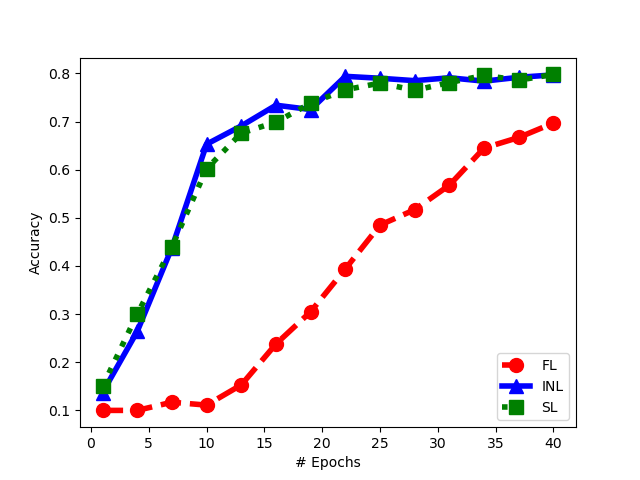}
			\caption{Accuracy vs. $\#$ of epochs.}
			\label{fig:accvsepochexp1}
		\end{subfigure}
		\hfill
		\begin{subfigure}{0.7\linewidth}
			\centering
			\includegraphics[width=0.8\linewidth]{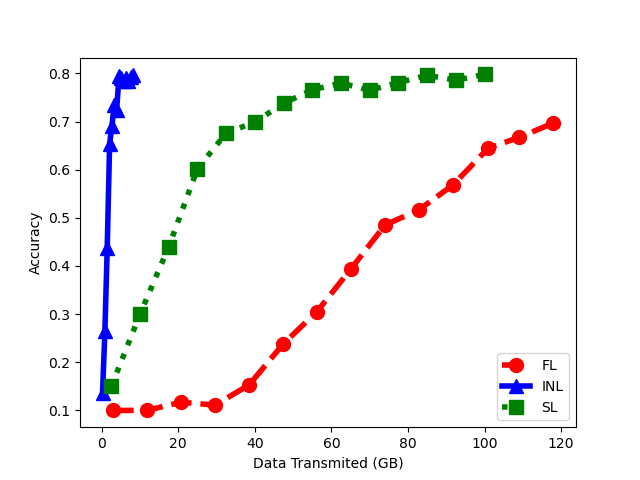}
			\caption{Accuracy vs. bandwidth cost.}
			\label{fig:accvsdataexp1}
		\end{subfigure}
		\caption{Comparison of INL, FL and SL - Experiment 1}
		\label{fig:experiment_results_1}
		\vspace{-0.2cm}
	\end{figure}
	
	For FL, each of the five client nodes is equipped with the \textit{entire} network of Figure~\ref{fig:nt_arh_5}. The dataset is split into five sets of equal sizes; and the split is now performed such that all five noisy versions of a same CIFAR-10 image are presented to the same client NN (distinct clients observe different images, however). For SL of~\cite{gupta2018distributed}, each input node is equipped with an NN formed by \textit{all} fives branches with convolution networks (i.e., all the network of Fig.~\ref{fig:nt_arh_5}, except the part at Node $(J+1)$); and node $(J+1)$ is equipped with fully connected layers at Node $(J+1)$ in Figure~\ref{fig:nt_arh_5}. Here, the processing during training is such that each input NN concatenates vertically the outputs of all convolution layers and then passes that to node $(J+1)$, which then propagates back the error vector. After one epoch at one NN, the learned weights are passed to the next client, which performs the same operations on its part of the dataset.
	

	Figure~\ref{fig:accvsepochexp1} depicts the evolution of the classification accuracy on CIFAR-10 as a function of the number of training epochs, for the three schemes. As visible from the figure, the convergence of FL is relatively slower comparatively. Also the final result is less accurate. Figure~\ref{fig:accvsdataexp1} shows the amount of data needed to be exchanged among the nodes (i.e., bandwidth resources) in order to get a prescribed value of classification accuracy. Observe that both our INL and SL require significantly less data exchange than FL; and our INL is better than SL especially for small values of bandwidth.
	
	\vspace{-0.2cm}
	\subsection{Experiment 2}
	
	\begin{figure}[!t]
		\centering
		\includegraphics[width=0.9\linewidth]{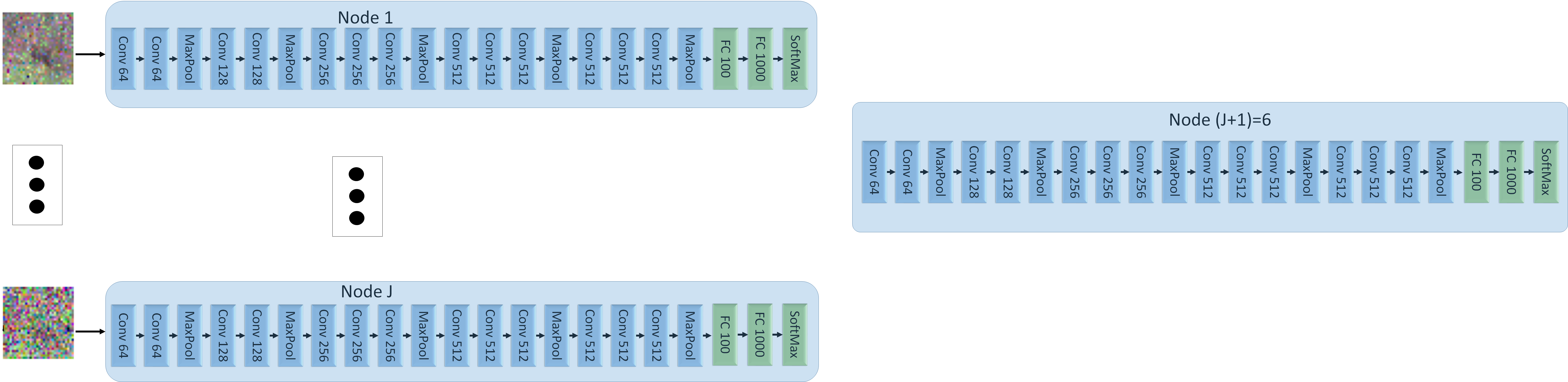}
		\caption{Used NN architecture for FL in Experiment 2} 
		\label{fig:accvsdata}
		\vspace{-0.2cm}
	\end{figure}
	
	\begin{figure}[!b]
		\centering
		\begin{subfigure}{0.7\linewidth}
			\centering
			\includegraphics[width=0.8\linewidth]{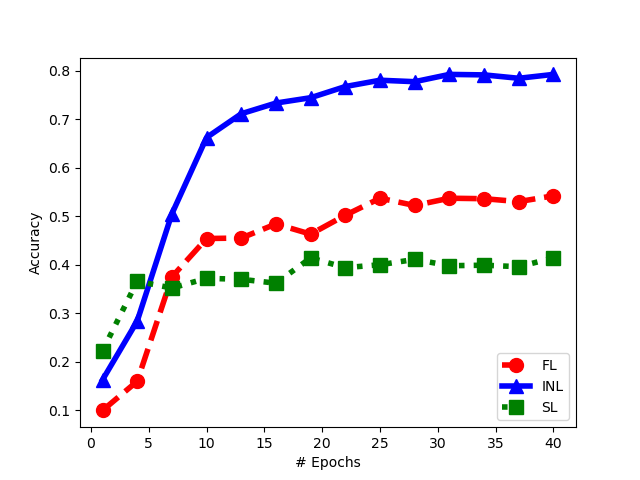}
			\caption{Accuracy vs. $\#$ of epochs.}
			\label{fig:accvsepochexp2}
		\end{subfigure}
		\hfill
		\begin{subfigure}{0.7\linewidth}
			\centering
			\includegraphics[width=0.8\linewidth]{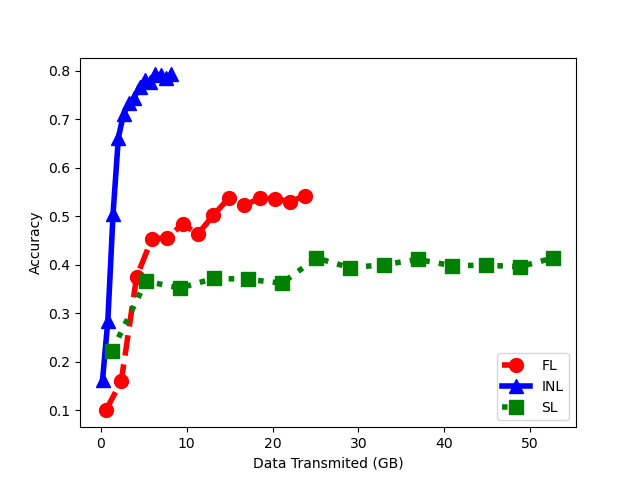}
			\caption{Accuracy vs. bandwidth cost.}
			\label{fig:accvsdataexp2}
		\end{subfigure}
		\caption{Comparison of INL, FL and SL - Experiment 2.}
		\label{fig:experiment_results_2}
		\vspace{-0.2cm}
	\end{figure}
	In Experiment 1, the entire training dataset was partitioned differently for INL, FL and SL (in order to account for the particularities of the three). In this second experiment, they are all trained on the same data. Specifically, each client NN sees all CIFAR-10 images during training; and its local dataset differs from those seen by other NNs only by the amount of added Gaussian noise (standard deviation chosen respectively as $0.4, 1, 2, 3, 4$). Also, for the sake of a fair comparison between INL, FL and SL the nodes are set to utilize fairly the same NNs for the three of them (see, Fig.~\ref{fig:accvsdata}).

	Figure~\ref{fig:accvsdataexp2} shows the performance of the three schemes during the inference phase in this case (for FL the inference is performed on an image which has average quality of the five noisy input images for INL and SL). Again, observe the benefits of INL over FL and SL in terms of both achieved accuracy and bandwidth requirements.

	\vspace{-0.2cm}
	\appendix
	
\renewcommand{\theequation}{A-\arabic{equation}}
\setcounter{equation}{0}  
\subsection{Proof of Theorem \ref{thorem_achive}}\label{appendix-proof-capacity-relavance-region}

The proof of Theorem 1 is based on a scheme in which the observations $\{\dv x_j\}_{j \in \mc J}$ are compressed distributively using Berger-Tung coding~\cite{berger_tung_coding}; and, then, the compression bin indices are transmitted as independent messages over the network $\mc G$ using linear-network coding~\cite[Section 15.4]{elgamalkim2011}. The decision node $N$ first decompresses 
the compression codewords and then uses them to produce an estimate $\hat{Y}$ of $Y$. In what follows, for simplicity we set the time-sharing random variable to be a constant, i.e., $Q=\emptyset$. Let $0 < \epsilon^{''} < \epsilon^{'} < \epsilon$. 

\subsubsection{Codebook Generation} Fix a joint distribution $P_{X_1,\hdots,X_J,Y,U_1,\hdots,U_J}$ that factorizes as given by~\eqref{joint-measure-statement-theorem1}. Also, let $D=H(Y|U_1,\hdots,U_J)$; and, for $(u_1,\hdots,u_J) \in \mc U_1 \times \hdots \times \mc U_J$, the reconstruction function $\hat{y}(\cdot | u_1,\hdots,u_J) \in \mc P(\mc Y)$ such that $\mathbb{E}\left[d(Y,\hat{Y})\right] \leq \dfrac{D}{1+\epsilon}$, where $d: \mc Y \times \mc P(\mc Y) \longrightarrow \mathbb{R}_{+}$ is the distortion measure given by~\eqref{definition-log-loss-distorsion-measure}. For every $j \in \mc J$, let $\tilde{R}_j \geq R_j$. Also, randomly and independently generate $2^{n\tilde{R}_j}$ sequences $u_j^n(l_j)$, $l_j \in [1:2^{n\tilde{R}_j}]$, each according to $\prod_{i=1}^n p_{U_j}(u_{ji})$. Partition the set of indices $l_j \in 2^{n\tilde{R}_j}$ into equal size bins $B_j(m_j)=\left[(m_j-1)2^{n\tilde{R}_j-R_j}:m_j2^{n\tilde{R}_j-R_j}\right]$, $m_j \in [1: 2^{nR_j}]$. The codebook is revealed to all source nodes $j \in \mc J$ as well as to the decision node $N$, but not to the intermediary nodes.

\subsubsection{Compression of the observations} Node $j \in \mc J$ observes $x_j^n$ and finds an index $l_j \in[1: 2^{n\tilde{R}_j}]$ such that $(x_j^n,u_j^n(l_j)) \in \mc T^{(n)}_{\epsilon^{''}}$. If there is more than one index the node selects one at random. If there is no such index, it selects one at random from $[1: 2^{n\tilde{R}_j}]$. Let $m_j$ be the index of the bin that contains the selected $l_j$, i.e., $l_j \in \mc B_j(m_j)$.

\subsubsection{Transmission of the compression indices over the graph network} In order to transmit 
the bins indices $(M_1,\hdots,M_J) \in [1:2^{nR_1}] \times \hdots \times [1:2^{nR_J}]$ to the decision node $N$ over the  graph network $\mc G=(\mc E, \mc N, \mc C)$, they are encoded as if they were independent-messages using the linear network coding scheme of~\cite[Theorem 15.5]{elgamalkim2011}; and then transmitted over the network. The transmission of the multimessage $(M_1,\hdots,M_J) \in [1:2^{nR_1}] \times \hdots \times [1:2^{nR_J}]$ to the decision node $N$ is without error as long as for all $\mc S \subseteq [1:N-1]$ we have 
\begin{equation}
	\sum_{j \in \mc S \cap \mc J} R_j \leq C(\mc S)
\end{equation}
where $C(\mc S)$ is defined by~\eqref{definition-cut-set}.

\subsubsection{Decompression and estimation} The decision node $N$ first looks for the unique tuple $(\hat{l}_1,\hdots,\hat{l}_J)\in \mc B_1(m_1)\times \hdots \times \mc B_J(m_J)$ such that $(u_1^n(\hat{l}_1),\hdots,u_J^n(\hat{l}_J)) \in \mc T^{(n)}_{\epsilon}$. With high probability, Node $N$ finds such a unique tuple as long as $n$ is large and for all $\mc S \subseteq \mc J$ it holds that~\cite{berger_tung_coding} (see also~\cite[Theorem 12.1]{elgamalkim2011}) 
\begin{equation}
	\sum_{j \in \mc S} R_j \geq I(U_{\mc S};X_{\mc S}|U_{\mc S ^c}).
	\label{inequalities-Berger-Tung-coding}
\end{equation}

\noindent The decision node $N$ then produces an estimate $\hat{y}^n$ of $y^n$ as $\hat{y}(u_1^n(\hat{l}_1),\hdots,u_J^n(\hat{l}_J))$.

 \noindent It can be shown easily that the per-sample relevance level achieved using the described scheme is $\Delta = I(U_1,\hdots,U_J;Y)$; and this completes the proof of Theorem~\ref{thorem_achive}.
\renewcommand{\theequation}{B-\arabic{equation}}

\setcounter{equation}{0}  
\subsection{Proof of Proposition~\ref{proposition-parametrization-region-theorem1}}\label{appendix-proof-parametrization-region-theorem1}

For $ C_{sum}\geq 0$ fix $s \geq 0$ such that $C_s=C_{sum}$; and let $\bm P^*=\lbrace P_{U^*_1|X_1},P_{U^*_2|X_2}, P_{U^*_3|X_3}\rbrace$ be the solution to \eqref{loss_function_hops_5} for the given $s$. By making the substitution in \eqref{region-5-node-sum-delta_s-c_s}: 
\begin{align}
	\Delta_s=&I(Y;U^*_1,U^*_2,U^*_3)\\
	\leq&\Delta\label{eq_proof_bound_part1}
\end{align}
where \eqref{eq_proof_bound_part1} holds since $\Delta$ is the maximum $I(Y;U_1,U_2,U_3)$ over all distribution for which \eqref{eq_csum_ineq_def} holds, which includes $\bm P^*$.\\
Conversely let $\bm P^*$ be such that $(\Delta,C_{\text{sum}})$ is on the bound of the $\mc{RI}_{\text{sum}}$ then:
\begin{align}
	\Delta=&H(Y)-H(Y|U^*_1,U^*_2,U^*_3)\nonumber\\
	\leq&H(Y)-H(Y|U^*_1,U^*_2,U^*_3)+sC_{\text{sum}}\nonumber\\
	&-s\big[I(X_2,X_3;U^*_2,U^*_3|U^*_1)+I(X_1,X_2,X_3;U^*_1,U^*_2,U^*_3)\big]\label{def_c_sum}\\
	\leq& H(Y)+\max_{\mathbf P}\mc L_s(\mathbf P) + sC_{\text{sum}}\label{eq:max_prop}\\
	=& \Delta_s-sC_s+sC_{\text{sum}}\nonumber\\
	=& \Delta_s+s(C_{\text{sum}}-C_s).\label{final_eq}
\end{align}
Where ~\eqref{def_c_sum} follows from \eqref{eq_csum_ineq_def}. Inequality \eqref{eq:max_prop} holds due to the fact that $\max_{\mathbf P} \mc L(\mathbf P)$ takes place over all $\bm P$, including $\bm P^*$. Since~\eqref{final_eq} is true for any $s\geq 0$ we take $s$ such that $C_{\text{sum}}=C_s$, which implies $\Delta \leq \Delta_s$. Together with \eqref{eq_proof_bound_part1} this completes the proof.

\renewcommand{\theequation}{C-\arabic{equation}}
\setcounter{equation}{0}  
\subsection{Proof of Lemma \ref{lemma-low-bound}}\label{appendix-proof-lower-bound}
We have
\begin{align}
	\mc L_s(\mathbf P) =& -H(Y|U_1,U_2,U_3)-sI(X_1,X_2,X_3;U_1,U_2,U_3)\nonumber\\
	&-sI(X_2,X_3;U_2,U_3|U_1)\\
	=&-H(Y|U_1,U_2,U_3)\nonumber\\
	&-s\Bigg[ I(X_1;U_1)+2I(X_2,X_3;U_2,U_3|U_1)\Bigg]\label{loss_eq_1_ex}\\
	=&-H(Y|U_1,U_2,U_3)-sI(X_1;U_1)-2sI(X_2;U_2)\nonumber\\
	&-2s\bigg[I(X_3;U_3)-I(U_3;U_1,U_2)-I(U_2;U_1)\bigg]\label{loss_eq_3_ex}\\
	=&-H(Y|U_1,U_2,U_3)-sI(X_1;U_1)-2sI(X_2;U_2)\nonumber\\
	&+2s\Big[I(U_2;U_1)+I(U_3;U_1,U_2)-I(X_3;U_3)\Big]\label{loss_eq_5_ex}\\
	\geq&-H(Y|U_1,U_4)-sI(X_1;U_1)-2s\Big[I(X_2;U_2)+I(X_3;U_3)\Big]\nonumber\\
	&+2s\Big[ I(U_2;U_1)+I(U_3;U_1,U_2)\Big]\label{loss_eq_6_ex}
\end{align}
where~\eqref{loss_eq_1_ex} holds since $U_1 \mkv X_1 \mkv (X_2,X_3,U_2,U_3)$ and $(U_2,U_3) \mkv (X_2,X_3) \mkv (U_1,X_1)$;~\eqref{loss_eq_3_ex} holds since $U_2 \mkv X_2 \mkv (U_1,X_3)$ and $U_3 \mkv X_3 \mkv (U_1,U_2,X_2)$; ~\eqref{loss_eq_6_ex} hold since $U_4 \mkv (U_2,U_3) \mkv (Y,U_1)$.
\renewcommand{\theequation}{D-\arabic{equation}}
\setcounter{equation}{0}  

\subsection{Proof of Lemma \ref{lemma-v_low_bound} } \label{appendix-proof-variational-lower-bound}
From \cite[eq. (55)]{aguerri2019distributed} it can be shown that for any pmf $Q_{Y|Z}(y|z)$ , $y \in \mc Y$ and $z \in \mc Z$ the conditional entropy $H(Y|Z)$ is :
\begin{equation}
	H(Y|Z)=\mathbb{E}[-\log Q_{Y|Z}(Y|Z)]- D_{\mathrm{KL}}(P_{Y|Z}||Q_{Y|Z}).
	\label{eq-var-bound-cond-entropy}
\end{equation}
And from \cite[eq. (81)]{aguerri2019distributed}:
\begin{align}
	I(X;Z) &= H(Z)-H(Z|X)\nonumber\\
	&=D_{\mathrm{KL}}(P_{Z|X}\Vert Q_{Z}) -D_{\mathrm{KL}}(P_{Z}\Vert Q_{Z}).\label{eq-var-bound-mutual-info}
\end{align}
Now substituting Equations \eqref{eq-var-bound-cond-entropy} and \eqref{eq-var-bound-mutual-info} in \eqref{loss_prob_function_hops_5} the following result is obtained:
\begin{align}
	\mc L_s^{\text{low}}(\bm P_+)=&-H(Y|U_1,U_4)-sI(X_1;U_1)-2sI(X_2;U_2)\nonumber\\
	&-2sI(X_3;U_3)+2s\Big[ I(U_2;U_1)+I(U_3;U_1,U_2)\Big]\nonumber\\
	=&\mathbb{E}[\log Q_{Y|U_1,U_4}]+D_{\mathrm{KL}}(P_{Y|U_1,U_4}||Q_{Y|U_1,U_4})\nonumber\\
	&-sD_{\mathrm{KL}}(P_{U_1|X_1}\Vert Q_{U_1})+sD_{\mathrm{KL}}(P_{U_1}\Vert Q_{U_1})\nonumber\\
	&-2sD_{\mathrm{KL}}(P_{U_2|X_2}\Vert Q_{U_2}) +2sD_{\mathrm{KL}}(P_{U_2}\Vert Q_{U_2})\nonumber\\
	&-2sD_{\mathrm{KL}}(P_{U_3|X_3}\Vert Q_{U_3}) +2sD_{\mathrm{KL}}(P_{U_3}\Vert Q_{U_3})\nonumber\\
	&+2sD_{\mathrm{KL}}(P_{U_2|U_1}\Vert Q_{U_2}) -2sD_{\mathrm{KL}}(P_{U_2}\Vert Q_{U_2})\nonumber\\
	&+2sD_{\mathrm{KL}}(P_{U_3|U_1,U_2}\Vert Q_{U_3}) -2sD_{\mathrm{KL}}(P_{U_3}\Vert Q_{U_3})\nonumber\\
	=&\mc L_s^{\text{v-low}}+sD_{\mathrm{KL}}(P_{U_1}\Vert Q_{U_1})+2sD_{\mathrm{KL}}(P_{U_2|U_1}\Vert Q_{U_2})\nonumber\\
	&+2sD_{\mathrm{KL}}(P_{U_3|U_1,U_2}\Vert Q_{U_3})+D_{\mathrm{KL}}(P_{Y|U_1,U_4}||Q_{Y|U_1,U_4})\nonumber\\
	\geq&\mc L_s^{\text{v-low}} \label{eq-ineq-var-lower-bound}
\end{align}
The last inequality~\eqref{eq-ineq-var-lower-bound} holds due to the fact that KL divergence is always positive and $s\geq 0$, thus proving the lemma.

	\vspace{-0.1cm}
	\bibliographystyle{IEEEtran}
	\bibliography{IEEEabrv,references} 	
\end{document}